\documentclass[final,leqno]{siamltex}
   
\usepackage{amssymb,amsmath}
\usepackage{mathrsfs,bm}
\usepackage{enumitem} 

\usepackage{fourier} 
\usepackage{oldgerm}
\usepackage[caption=false]{subfig}
\usepackage[final]{graphicx} % must come before mathtools to pass options
\usepackage{graphicx}

\newcommand{\abs}[1]{\left|#1\right|}
\newcommand{\norm}[1]{\left\|#1\right\|}

\usepackage[dvipsnames]{xcolor}

\usepackage{tikz} % Necessary to create the bloc diagrams
\usetikzlibrary{shapes,arrows} 

\usepackage{pgfplots}
\newlength\figureheight 
\newlength\figurewidth

\def\RR{\mathbb{R}}

\def\NN{\mathbb{N}}

\def\V#1{{\bm #1}}          % vector
\def\ue{\mathrm{e}}

\def\uj{\mathrm{j}}

\def\Re{\mathop{\mathrm{Re}}\nolimits}

\begin{document}

\author{
Zsuzsanna P\"usp\"oki
\thanks{Biomedical Imaging Group, Ecole polytechnique f\'ed\'erale de Lausanne
(EPFL), CH-1015, Lausanne, Switzerland  ({\tt first.last@epfl.ch}). \newline
Zsuzsanna P\"usp\"oki and John Paul Ward have an equal contribution to the paper.}
\and
John Paul Ward 
\footnotemark[2]
\and
Daniel Sage
\footnotemark[2]
\and
Michael Unser 
\footnotemark[2]
}

\title{On The Continuous Steering of the Scale of Tight Wavelet Frames
\thanks{The research leading to these results has received funding from the
European Research Council under the European Union's Seventh Framework
Programme (FP7/2007-2013) / ERC grant agreement $\text{n}^\circ$ 267439.
The work was also supported by the Hasler Foundation. }}

\maketitle 
 
\begin{abstract}
In analogy with steerable wavelets, we present a general construction of
adaptable tight wavelet frames, with an emphasis on scaling operations.  In
particular, the derived wavelets can be ``dilated'' by a procedure comparable to
the operation of steering steerable wavelets. The fundamental aspects
of the construction are the same: an admissible collection of Fourier
multipliers is used to extend a tight wavelet frame, and the ``scale'' of the
wavelets is adapted by scaling the multipliers.  As an application, the
proposed wavelets can be used to improve the frequency localization. Importantly, the localized frequency bands specified by this construction can
be scaled efficiently using matrix multiplication.
%and this construction can be viewed as an analog of the steerable wavelet
% construction.
\end{abstract}



\pagestyle{myheadings}
\thispagestyle{plain}
\markboth{Z.  P\"USP\"OKI, J.~P. WARD, D. SAGE AND M. UNSER}{ON THE CONTINUOUS STEERING OF THE SCALE OF TIGHT WAVELET FRAMES}

\section{Introduction}

The representation of objects as the combination of simpler building blocks
is fundamental to many applications, including signal and image processing.
The advantage of methods that exploit this idea is twofold. First, they often
provide a certain measure of locality that allows us to isolate information
related to a particular aspect (e.g., frequency or spatial location).
Second, an object can be represented in some transform domain with a small
number of coefficients; thus, we can compress the representation or summarize the
data efficiently by keeping only the most-important coefficients (or a sparse
distribution of them).

Two somewhat contradictory strategies for generating local representations are: 1)
using a generic framework and universal methods (e.g., Fourier
analysis or wavelets), or 2) working with specific, signal-adaptive
methods (e.g., local principal/independent component analysis). Our interest is
in representations that lie between these extremes, where the predefined
building blocks can be adapted to the local information by applying a
transformation. Thus, we maintain generality and universal performance while
capturing specific information.

In the context of image processing, local geometric structures are
often repeated throughout natural images; however, each occurrence has typically
been deformed by an unknown geometric transformation such as a combination of 
rotation, translation, or scaling. Similar observations can be made for many
 types of signals, for instance those encountered in bioengineering, seismology,
audio, and music processing, to name a few. This motivates us to combine local
representations with adaptable transformations. We achieve this goal by looking for 
transformations within the family of wavelet transforms.

A general approach to constructing adaptable wavelets was presented in
\cite{unser13,ward13_sw}, with the primary focus being steerable
wavelets and rotationally invariant properties.
The origin of steerable wavelets can be traced back to the work of Freeman and
Adelson \cite{93808} on steerable filters.  Simoncelli et al.\ extended this
work by considering the operations of translation and scaling
\cite{SFAH92}.  As a further generalization, Perona considered
arbitrary compact transformations without requiring the group property
\cite{Per95}.  Then, in his doctoral dissertation and a series of papers, Teo
unified the existing theories and provided a solid mathematical framework for
the study of transformability based on the theory of Lie groups and Lie
algebras \cite{Teo98, Teo99}.
In this paper, we combine transformable filters with primal wavelet
systems to form adaptable wavelet frames, focusing on
scaling properties.

\subsection{Roadmap}

In Section \ref{sec:math}, we cover notations that will be used throughout
the paper and present the mathematical
foundations of our wavelet construction. Specifically, we cover transformable
filters, with a focus on scaling, and some fundamental wavelet results. According to the terminology introduced by Wilson and Knutsson \cite{WK88}, a collection of 
transformable filters corresponds to a basis for an equivariant
space.  Also in Section \ref{sec:math}, we define 
admissible collections of Fourier multipliers (filters) as a special class of
transformable families.  We then show how an admissible collection can be used
to extend a tight frame of $L_2 \left( \mathbb{R}^d \right)$ to a new tight
frame.
An important point of this construction is that we have some freedom in the choice
of the multipliers, which allows us to shape the Fourier transforms of the
primal tight frame.  The benefit of this construction is that it allows us to
extract physical features from a set of data.  
In Section \ref{sec:scaling_families}, we focus on families of
multipliers whose span is invariant to scaling.  Then, in Section
\ref{sec:steer}, we explain the operation of transforming an extended tight
wavelet frame.  Essentially, the coefficients corresponding to the extended
wavelet frame can be used to compute coefficients for the wavelet frame
generated by a scaled version of the multipliers.  In Section
\ref{sec:freq_loc}, we present a particular construction that focuses on
refining the frequency localization of the primal-wavelet system. In Section
\ref{sec:two_channel}, we consider a two-channel example and illustrate the
proposed wavelets. In Section \ref{sec:application}, we use our wavelet design to detect circular patters. In 
Section \ref{sec:experiments}, we evaluate our results and make comparisons with other popular methods.

\section{Mathematical Formulation}
\label{sec:math}

In this section, we cover the mathematical structure of our wavelet
construction.  We begin with notations and then introduce general results concerning the extension of a
tight frame, before moving on to the details of admissible Fourier
multipliers and transformable families. We also discuss how to adapt the
extended frame and provide information on the construction of a particular primal-wavelet frame.

\subsection{Notation}

The variable $\bm{x}$ represents a point in the spatial domain
$\mathbb{R}^d$.  A point in the Fourier domain is denoted in Cartesian
coordinates as $\bm{\omega}$. The variables $\rho$ and $\bm{\theta}$ are
used to represent $\bm{\omega}\neq \bm{0}$ in polar coordinates, with
\begin{align}
\rho &= \abs{\bm{\omega}}\\
\bm{\omega} &= \rho \bm{\theta}.
\end{align}

The Fourier transform of an $L_1 \left(\mathbb{R}^d \right)$ function $f$ is 
\begin{equation}
\widehat{f}(\bm{\omega}) 
= 
\int_{\mathbb{R}^d} f(\bm{x})
\ue^{ -\uj\left<\bm{x},\bm{\omega}\right>} {\rm d}\bm{x}.
\end{equation}

We use the notation $\log_2$ to denote the base-two logarithm.

\subsection{Operators}

The two components of our construction are a primal-wavelet frame and
a collection of bounded linear operators on $L_2$.  Since we are working with a
wavelet system, we require these operators to commute with translations. 
Every operator of this form can be described as a Fourier multiplier in $L_\infty \left( \mathbb{R}^d \right)$ \cite[Theorem
3.18]{stein71}. Specifically, for each such operator $T$, there is a symbol
$\widehat{T}\in L_\infty \left( \mathbb{R}^d \right)$ such that
\begin{equation} 
\mathcal{F} \left\{ T f \right\} 
= 
\widehat{T} \widehat{f},
\end{equation}
for every $f\in L_{2}(\mathbb{R}^d)$.
A collection of Fourier multipliers is said to be admissible if it satisfies the partition-of-unity property below.

\begin{definition}[cf. {\cite[Definition 2.1]{ward13_rd}}]
A collection of complex-valued functions
$\mathcal{M} = \left\{ M_n \right\}_{n=1}^{ n_{\mathrm{max}} }$ is admissible if
\begin{enumerate}
  \item Each $M_n$ is Lebesgue-measurable;
  \item The squared moduli of the elements of $\mathcal{M}$ form the partition of
  unity
  \begin{equation}
  \sum_{n=1}^{n_{\mathrm{max}}} \abs{M_n(\bm{\omega})}^2
  =
  1
  \end{equation}
  for every $\bm{\omega}\in \mathbb{R}^d\backslash \{\bm{0}\}$.
\end{enumerate}
\end{definition}

The partition-of-unity property is important as it allows us to use an
admissible collection to extend a tight frame to a new tight frame.

\begin{theorem}[cf. {\cite[Theorem 2.4]{ward13_rd}}]
\label{thm:extended_frame} Suppose
$ \left\{ \phi_{k}\right\}_{k \in \mathbb{Z}}$ is a  tight
frame of $L_2 \left( \mathbb{R}^d \right)$, with
\begin{equation}
f
=
\sum_k \left\langle f,\phi_{k} \right\rangle \phi_{k}
\end{equation}
and
\begin{equation}
\norm{f}_{L_2}^2
=
\sum_k\abs{ \left\langle f,\phi_{k}  \right\rangle }^2
\end{equation}
for every $f\in L_2 \left( \mathbb{R}^d \right)$. Also, let
$\mathcal{M} = \left\{ M_n \right\}_{n=1}^{ n_{\mathrm{max}} }$ be admissible.
Then, the collection
\begin{equation}
\left\{
\mathcal{F}^{-1}
\left\{ M_{n}\widehat{\phi}_{k} \right\} \right\}
_{n=1, \dots, n_{\mathrm{max}};  k\in\mathbb{Z} }
\end{equation}
is also a tight frame.
\end{theorem}
% \begin{proof}
% The tight frame property of the original frames implies
% \begin{align*}
% \norm{\bar{M}_{j,n}\widehat{f}}_{L_2}^2 &= \sum_k
% \abs{\left<\bar{M}_{j,n}\widehat{f} ,\widehat{\phi}_{j,k} \right>}^2\\
% &= \sum_k
% \abs{\left<\widehat{f} ,M_{j,n}\widehat{\phi}_{j,k} \right>}^2.
% \end{align*}
% Summing over $j$ and $n$ proves the result. 
% 
% \end{proof}
 
 As in the case of steerable wavelets, we can adapt (or shape) the multipliers
 while maintaining the tight-frame property if we apply an isometry to the
 collection.  The importance of this construction is prominent in Section
 \ref{sec:scaling_families}, where we combine a collection of individual
 trigonometric functions to build trigonometric polynomials.
 \begin{proposition}\label{pr:isometry}
 Let $\left\{ M_n \right\}_{n=1}^{ n_{\mathrm{max}} }$ be an admissible
 collection and define the vector $\mathbf{M}$ to have entries
 $\left[ \mathbf{M} \right]_n=M_n$.  If the $(n_{\mathrm{max}}^{\prime} \times
 n_{\mathrm{max}})$ matrix $\mathbf{U}$ is an isometry (i.e., $\mathbf{U}^*
 \mathbf{U}$ is the $(n_{\mathrm{max}} \times n_{\mathrm{max}})$ identity matrix),
 then the elements of $\mathbf{UM}$ form an admissible collection of size
 $ n_{\mathrm{max}}^{\prime} $.
 \end{proposition}
 \begin{proof}
 This follows immediately from the definitions of $\mathbf{M}$ and $\mathbf{U}$.
 \end{proof}

\subsection{Admissible Multipliers and Transformable Families}

Here, we briefly review the idea of transformability in the context of rotation
and then show how it extends to the general setting. 

In two dimensions, an example of an admissible collection of homogeneous multipliers is 
$\mathcal{M} = \{M_1, M_2 \}$, where
\begin{align}
	M_1(\bm{\omega}) 
	&= 
	\sin(\xi) \\
	M_2(\bm{\omega}) 
	&= 
	\cos(\xi)
\end{align}
and $\xi$ denotes the angle that $\bm{\omega}$ makes with a fixed direction. Any rotation of these functions can be written as a weighted sum of the
unrotated functions, as in 
\begin{equation}
	\left(
	\begin{matrix}
	\cos(\xi+\xi_0) \\
	\sin(\xi+\xi_0)
	\end{matrix}
	\right)
	=
	\left(
	\begin{matrix}
	\cos(\xi_0) &  -\sin(\xi_0)\\
	\sin(\xi_0) &  \phantom{+}\cos(\xi_0)
	\end{matrix}
	\right)
	\left(
	\begin{matrix}
	\cos(\xi)\\
	\sin(\xi)
	\end{matrix}
	\right).
\end{equation}
Now suppose that we have a tight frame $\left\{ \phi_k \right\}$, which is
extended by $\mathcal{M}$ as described in Theorem \ref{thm:extended_frame}, and let $f\in
L_2(\mathbb{R}^2)$.  Then the rotated multipliers 
\begin{equation}
	\mathcal{M}_{\xi_0}
	=
	\{\cos(\cdot+\xi_0), \sin(\cdot+\xi_0)\}
\end{equation}
are also admissible and can be used to extend the frame.  Importantly, we can
use the frame coefficients from one frame to compute the coefficients for the other by
\begin{equation}
\left(
\begin{matrix}
\left<\widehat{f},\widehat{\phi}_k \cos(\xi+\xi_0)\right> \\
\left<\widehat{f},\widehat{\phi}_k \sin(\xi+\xi_0)\right>
\end{matrix}
\right)
=
\left(
\begin{matrix}
\cos(\xi_0) &  -\sin(\xi_0)\\
\sin(\xi_0) &  \phantom{+}\cos(\xi_0)
\end{matrix}
\right)
\left(
\begin{matrix}
\left<\widehat{f},\widehat{\phi}_k \cos(\xi)\right>\\
\left<\widehat{f},\widehat{\phi}_k \sin(\xi)\right>
\end{matrix}
\right).
\end{equation}

Steerable wavelets are constructed using a tight wavelet frame
$\{\phi_k\}$, where the mother wavelet is bandlimited and isotropic. Notice in
particular that the rotation invariance of the primal wavelets means that the
derived wavelets themselves (not only the multipliers) are rotated by matrix
multiplication. This construction has been used to capture the local orientation of features in
images \cite{unser13}.

The steerable-wavelet construction is based on the group of rotation
operators on $L_2(\mathbb{R}^d)$; however, an analogous construction can be
performed in a more general setting.
\begin{definition}
Let $\{ G_a \}_{a \in A}$ be a group of bounded linear operators on
$L_2 \left( \mathbb{R}^d \right)$, indexed by a set $A$. An admissible family of
multipliers $\mathcal{M} = \left\{ M_n \right\}_{n=1}^{n_{\mathrm{max}}}$ is
{\bfseries transformable} by $\{ G_a \}_{a \in A}$ if the span of $\mathcal{M}$
is invariant under the action of any operator from $\{ G_a \}$.
\end{definition}

% ??? connect to Teo's work

This definition is a generalization of the idea of eigenfunctions. In
the simplest case, we have a single multiplier $\mathcal{M}=\{M\}$ that is an
eigenfunction for each transform $ G_a $ and there exist 
$\lambda_a \in \mathbb{R}$ such that
\begin{equation}
G_a M 
= 
\lambda_a M.
\end{equation}

For $\mathcal{M} = \left\{ M_n \right\}_{n=1}^{n_{\mathrm{max}}}$ containing
several functions, this condition is relaxed, as the eigenvalues are
replaced by matrices. For any $a\in A$, there must be a matrix
$\mathbf{\Lambda}_a$ such that
\begin{equation}
\left(
\begin{matrix}
G_a M_1 \\
\vdots \\
G_a M_{n_{\mathrm{max}}}
\end{matrix}
\right) 
=
\mathbf{\Lambda}_a
\left(
\begin{matrix}
M_1 \\
\vdots \\
M_{n_{\mathrm{max}}}
\end{matrix}
\right).
\end{equation}

When the span of the multipliers is invariant to
the action of an operator, the wavelet coefficients of the transformed tight
frame can be derived from the wavelet coefficients of the untransformed frame. 

\subsection{Families of Dilation Multipliers}\label{sec:scaling_families}

Here, we are interested in studying scaling properties, so we consider the
transformation group $\{G_a\}_{a\in (0,\infty)}$ acting on $L_2(\mathbb{R}^d)$,
where $G_a f = f(a\cdot)$ for any $f\in L_2(\mathbb{R}^d)$.  We restrict our
attention to admissible collections of multipliers $\{M_n\}$, where each
multiplier is radial, which means that there exists $m_n:(0,\infty)\rightarrow \mathbb{C}$
such that $M_n(\bm{\omega})=m_n(\abs{\bm{\omega}})$.  Therefore, we require
\begin{equation}\label{eq:scale_invariance}
\mathrm{span}\{m_n: n= 1, \dots, n_{\mathrm{max}} \}
= 
\mathrm{span}\{m_n(a\cdot) : n= 1, \dots, n_{\mathrm{max}} \}
\end{equation}
for any dilation $a\in(0,\infty)$.  Notice that we can write 
\begin{equation}
m_n(\cdot)
=
m_n \left(2^{\log_2(\cdot)}\right).
\end{equation} 
This defines the new function
$g_n:=m_n \left(2^{(\cdot)}\right)$ that should satisfy
\begin{equation}
\mathrm{span}\{g_n : n= 1, \dots, n_{\mathrm{max}}\}
=
\mathrm{span}\left\{g_n(\cdot+b) : n= 1, \dots, n_{\mathrm{max}} \right\}
\end{equation}
for any real $b$.  Hence the span of $\{g_n\}$ should be translation invariant. 
Conversely, any translation-invariant set gives rise to a dilation-invariant
one. This means that, once we identify all shiftable functions, we can
easily derive scalable families.  Finite-dimensional translation-invariant sets
of functions have been classified; the following results follows from
\cite{anselone64}.

\begin{proposition}
\label{prop:shift_form}
Let $\{g_1, \ldots, g_{n_{\mathrm{max}}} \}$ be a family of differentiable
functions on $\mathbb{R}$, and let $\mathbf{G}$ denote the vector with
$g_n$ as its $n$th component. Further, suppose that the collection is shiftable so that, for any $b\in \mathbb{R}$, there is a matrix $\mathbf{\Lambda}(b)$ such
that
\begin{equation}\label{eq:shift_1}
\mathbf{G}(\cdot + b)
= 
\mathbf{\Lambda}(b) \mathbf{G}. 
\end{equation}
If the components of $\mathbf{\Lambda}$ are differentiable functions,
then each $g_n$ is a linear combination of functions of
the form
\begin{equation}
x^k 2^{\alpha x} 
\end{equation}
for $x \in  \RR^+\setminus \{0\}, k \in \NN$, and $\alpha\in\mathbb{C}$. 
\end{proposition}

%\begin{proof}
% 
% We can see that the collection $\{g_1', \ldots, g_N' \}$ is shiftable by
% differentiating \eqref{eq:shift_1}:
% \begin{equation}
% \mathbf{G}'(\cdot-b)= \mathbf{\Lambda}(b) \mathbf{G}'. 
% \end{equation}
% Also, differentiating \eqref{eq:shift_1} with respect to $b$ gives
% \begin{equation}
% -\mathbf{G}'(\cdot-b)= \frac{\rm d}{{\rm d}b}\mathbf{\Lambda}(b) \mathbf{G}. 
% \end{equation}
% We now have
% \begin{equation}
% \mathbf{G}'  = \mathbf{\Lambda}(-b)\mathbf{G}'(\cdot-b)\\
% \end{equation}
% \begin{equation}\label{eq:shift_2}
% \mathbf{G}'= - \mathbf{\Lambda}(-b) \frac{\rm d}{{\rm d}b}\mathbf{\Lambda}(b)
% \mathbf{G}.
% \end{equation}
% As the left-hand side of \eqref{eq:shift_2} is independent of $b$, the
% right-hand side is as well.  Therefore, \eqref{eq:shift_2} is a homogeneous
% linear system of ordinary differential equations with constant coefficients, and
% the result follows.
%\end{proof} 

\begin{proposition}
\label{prop:sc_form}
Let $\{m_1,\dots,m_{n_{\mathrm{max}}}\}$ be a differentiable collection of
functions, and let $\mathbf{M}$ denote the vector with $m_n$ as its $n$th component.
Further, suppose that the collection is scalable so that, for any $a \in (0,\infty)$, there is a matrix $\mathbf{\Lambda}(a)$ that
satisfies
\begin{equation}
\mathbf{M}(a \cdot)
= 
\mathbf{\Lambda}(a) \mathbf{M}(\cdot). 
\end{equation}
If the components of $\mathbf{\Lambda}$ are differentiable functions,
then each $m_n$ is a linear combination of functions of the form
\begin{equation}\label{eq:scalable_1}
\left(\log_2 (x)\right)^k 2^{\alpha \log_2(x)} 
= 
\left(\log_2 (x)\right)^k x^{\alpha}
\end{equation}
for $x \in  \RR^+\setminus \{0\}, k \in \NN$, and $\alpha\in\mathbb{C}$. 
\end{proposition}

Multipliers are required to be bounded to form admissible collections. Therefore,
functions of the form \eqref{eq:scalable_1} with $k>0$ or with the real part of
$\alpha$ non-zero cannot be included.  Consequently, we consider admissible
collections composed of trigonometric polynomials.  A
particular example is given by the collection of trigonometric functions
\begin{equation}
\mathcal{M} = \left\{m_n (\log_2 \abs{\cdot} )  
\right\}_{n=1}^{n_{\mathrm{max}}},
\end{equation} 
where $n_{\mathrm{max}}$ is some fixed degree and where
\begin{equation}\label{eq:trig_multipliers}
\{m_n\}_{n=1}^{n_{\mathrm{max}}}
=
\{\alpha_0\}\bigcup_{l=1}^{l_{\mathrm{max}}}
\left\{\alpha_l \cos \left( \frac{2\pi l}{\sigma} \cdot\right),\alpha_l\sin
\left( \frac{2\pi l}{\sigma}\cdot\right) \right\},
\end{equation}
$\sum_{l=0}^{l_{\mathrm{max}}} \abs{\alpha_l}^2 = 1$, and $\sigma>0$.
Note that this collection is transformable due to the angle-addition formulas
\begin{align}\label{eq:angle_addition}
\sin(\rho_1+\rho_2)
&=
\sin(\rho_1)\cos(\rho_2)+\cos(\rho_1)\sin(\rho_2) \nonumber \\
\cos(\rho_1+\rho_2)
&=
\cos(\rho_1)\cos(\rho_2)-\sin(\rho_1)\sin(\rho_2).
\end{align}
Admissibility follows immediately from
the fact that $\sin^2+\cos^2=1$.

A general class of admissible collections of multipliers can be defined by
combining trigonometric functions into trigonometric polynomials; a
criterion for admissibility is given in \cite{unser13}.
 
\begin{proposition}\label{prop:trig_poly} For $\sigma>0$, the collection
$\mathcal{M} = \{ m(\log_2 (\rho_n \abs{\cdot}) ) \}_{n=1}^{n_{\mathrm{max}}}$,
with
\begin{align}
m(\rho)  
= 
\frac{\alpha_0}{\sqrt{{n_{\mathrm{max}}}}} +\sum_{l=1}^{l_{\mathrm{max}}}
\sqrt{\frac{2}{n_{\mathrm{max}}}} \alpha_l \cos\left( \frac{2\pi l}{\sigma}
\rho \right), \\
\rho_n 
= 
2^{\sigma n/{n_{\mathrm{max}}}}, \quad \mathrm{and} \quad   
\sum_{l=0}^{l_{\mathrm{max}}} \abs{\alpha_l}^2 =1,
\end{align}
is admissible if ${n_{\mathrm{max}}}\geq 2{l_{\mathrm{max}}}+1$.
\end{proposition}

\begin{proof}
This is a log mapping of Theorem 5.2 from \cite{unser13}.
\end{proof}

\subsection{Adapting Extended Frames}\label{sec:steer}

The scale-invariance property \eqref{eq:scale_invariance} is important
because it allows us to scale the multipliers in an admissible collection
$\mathcal{M}$
using matrix multiplications. When these multipliers are combined with a tight
frame to form an extended frame, the scale 
invariance means that, once we have computed the wavelet coefficients for the system
derived from $\left\{M_n(a\cdot)\right\}$, we can use matrix multiplications to
find the coefficients of the wavelet system derived from
$\left\{M_n(a'\cdot)\right\}$ for $a \neq a'$. Moreover, we can choose $a$
independently at each point.

To see how this works, suppose that we have at our disposal a normalized tight frame
$\{\phi_k\}_{k=1}^\infty$ and the scalable, admissible family of dilation
multipliers
\begin{equation}\label{eq:multipliers}
\left\{M_n\right\}_{n=1}^{n_{\mathrm{max}}} =  \left\{m_n \left(\log_2 \left(\abs{\cdot} \right) \right)
\right\}_{n=1}^{n_{\mathrm{max}}}.
\end{equation}
The collection
$\mathcal{M}_a=\{M_n(a\cdot)\}$ is also admissible for any $a>0$. Indeed,
\begin{equation}\label{eq:extended_frame_a}
\left\{\mathcal{F}^{-1}\{M_n(a\cdot)\widehat{\phi}_k\}:k\in\mathbb{Z},n=1,\dots,{n_{\mathrm{max}}}
\right\}
\end{equation} 
is a normalized tight frame. 

Now, let $\mathbf{M}_a$ be the matrix whose $n$th entry is $M_n(a\cdot)$. As
these collections are scalable, there are matrices $\mathbf{\Lambda}(a)$ such
that
\begin{equation}
\mathbf{M}_a 
= 
\mathbf{\Lambda}(a) \mathbf{M}.
\end{equation}
Therefore, knowing the frame coefficients of $f\in L_2(\mathbb{R}^d)$ for any
frame \eqref{eq:extended_frame_a} with $a>0$, we can easily compute the
coefficients for $a'\neq a$ by
\begin{equation}
\left(
\begin{matrix}
\left<\widehat{f},\widehat{\phi}_k M_1(a'\cdot)\right> \\
\vdots\\
\left<\widehat{f},\widehat{\phi}_k M_N(a'\cdot)\right>
\end{matrix}
\right)
=
\Lambda(a')\Lambda(a^{-1})
\left(
\begin{matrix}
\left<\widehat{f},\widehat{\phi}_k M_1(a\cdot)\right>\\
\vdots\\
\left<\widehat{f},\widehat{\phi}_k M_N(a\cdot)\right>
\end{matrix}
\right).
\end{equation}

As a particular example, suppose all multipliers are dilations of a single multiplier. Let $\mathcal{M}$ be
defined as in Proposition \ref{prop:trig_poly} with $ n_{\mathrm{max}} \geq
2 l_{\mathrm{max}} + 1$ and $\sigma = 2$. In this situation, we have that
\begin{equation}
\mathbf{M}_a 
= 
\mathbf{U} \mathbf{D}_a \mathbf{B},
\end{equation}
where $\mathbf{U}$ is the $\left( {n_{\mathrm{max}}} \times (2{l_{\mathrm{max}}} +1
) \right)$ matrix with entries
\begin{equation}
[\mathbf{U}]_{n,l} 
= 
\frac{1}{\sqrt{n_{\mathrm{max}}}}\ue^{\uj \pi l \log_2(\rho_n)}
\end{equation}
for $l=-{l_{\mathrm{max}}},\dots,{l_{\mathrm{max}}}$ and
$n=1,\dots,{n_{\mathrm{max}}}$; $\mathbf{B}$ is the vector with entries
\begin{equation}
[\mathbf{B}]_l = 
\begin{cases}
\frac{\alpha_{\abs{l}}}{\sqrt{2}} \ue^{\uj \pi l \log_2(\rho) }, \quad &l \neq 0\\
\alpha_{0}, \quad &l = 0
\end{cases}
\end{equation}
and $\mathbf{D}_a$ is the diagonal matrix with entries
\begin{equation}
[\mathbf{D}_a]_{l,l} 
=  
\ue^{\uj \pi l \log_2(a) }.
\end{equation}
The matrix $\mathbf{U}$ is an isometry, so for any $a'>0$, we have that
\begin{align}
\mathbf{M}_{a'} 
&= 
\mathbf{U} \mathbf{D}_{a'} \mathbf{B}\\
&= 
\mathbf{U} \mathbf{D}_{a'} \left( \mathbf{D}_{a^{-1}}  \mathbf{U}^T
\mathbf{M}_a \right) \\
&= 
\mathbf{U} \mathbf{D}_{a'a^{-1}} \mathbf{U}^T \mathbf{M}_a.
\end{align}

Therefore,  $\mathbf{T}_{a,a'}= \mathbf{U} \mathbf{D}_{a'a^{-1}}
\mathbf{U}^T$ is the matrix used to transform $\mathcal{M}_a$ into
$\mathcal{M}_{a'}$.

\subsection{Wavelets}

Our construction is initialized with a
tight wavelet frame of $L_2(\mathbb{R}^d)$ whose basis functions are generated
by dilations and translations of the single mother wavelet $\phi$. Proposition \ref{prop:isotropic} 
exhibits sufficient conditions for such a wavelet system.
\begin{proposition}\label{prop:isotropic}
Let $h:[0,\infty)\rightarrow \mathbb{R}$ be a smooth function that satisfies

\begin{enumerate}[label={\Roman*.}]
\item 
$h(\rho)=0$ for $\rho >\pi$  \quad (bandlimited)
\item 
${\displaystyle \sum_{q\in \mathbb{Z}}
\abs{h \left(2^q\rho \right)}^2=1}$
\item 
${\displaystyle \left. \frac{{\rm d}^nh}{{\rm
d}\rho^n}\right|_{\rho=0}=0}$ \quad for \quad $n=0,\dots,N$  \quad(vanishing
moments).
\end{enumerate}
For $1\leq p \leq \infty$, the mother wavelet $\psi$ whose
$d$-dimensional Fourier transform is given by
\begin{equation}
\widehat{\phi} \left(\bm{\omega} \right)
=
h \left(\norm{\bm{\omega}}_{\ell_p} \right)
\end{equation}
generates a normalized tight wavelet frame of $L_2(\mathbb{R}^d)$ whose basis
functions
\begin{equation}
\phi_{q,\bm{k}}(\bm{x})
=
\phi \left(\bm{x}-2^{q}\bm{k} \right)  
\end{equation}
have vanishing moments up to order $N$.  In particular, any $f \in L_2
\left(\mathbb{R}^d \right)$ can be represented as
\begin{equation}
f 
= 
\sum_{q\in\mathbb{Z}}\sum_{\bm{k}\in \mathbb{Z}^d}
\left\langle f,\phi_{q,\bm{k}}\right \rangle \phi_{q,\bm{k}}.
\end{equation}
\end{proposition} 
\begin{proof}
This follows from a combination of Parseval's identity for Fourier transforms
and Plancherel's identity for Fourier series.
\end{proof}

%{There are various types of wavelets satisfying Proposition \ref{prop:isotropic}. For the construction of scalable tight wavelet frames, we could use any of those. However, not all the frames are identical from an application point of view. When choosing the window function for the scaling of the multipliers, we have to consider two main issues. First, we would like to have an overlapping wavelet scheme, such that the transition between the dyadic scales are ''continuous''. Second, we would like to preserve the shape of the multipliers. }

In Section \ref{sec:freq_loc}, we use Meyer-type wavelets, constructed using the techniques of \cite{ward13_rd}.
% %
% \begin{figure}[h] 
%   \centering
%   %   \includegraphics[width=6cm]{fig/Simoncelli} 
%   \includegraphics[width=6cm]{fig/meyer_wave}
%   \caption[Meyer-type wavelet]
%    {Plot of $h_\epsilon$ approximating
%    $(2\pi)^{-d/2}2^{-1/2}\chi_{[\pi/4,\pi]}$.
%    }
%   \label{fig:meyer_wave}
% \end{figure}
% %
We want our final wavelets
to take the shape of the multipliers in an admissible family. Therefore,
we would ideally like to use a primal wavelet where $h$ is a constant
multiple of the characteristic function of $[\pi/4,\pi]$. However, discontinuities
in the Fourier domain correspond to slow decay in the spatial domain.  Therefore, as a
tradeoff, we propose a profile $h_\epsilon$ that is a smooth approximation to
the characteristic function of $[\pi/4,\pi]$, where $\epsilon$ is an
approximation parameter. For this construction, we define a smooth, non-decreasing
function $G$ that satisfies
$G(\gamma)=0$ for $\gamma<-1$ and $G(\gamma)=\pi/2$ for $\gamma>1$. Then, for an
approximation parameter $\epsilon$, we define
\begin{equation}\label{eq:mey_wave_01}
H_\epsilon(\gamma) 
=
G\left(\frac{\gamma+1}{\epsilon}\right)-\frac{\pi}{2}+G\left(\frac{\gamma-1}{\epsilon}\right)
\end{equation}
and 
\begin{equation}\label{eq:mey_wave_02}
h_\epsilon(\gamma) 
=
2^{-1/2}\cos
\left(H_\epsilon\left(\log\left(\frac{2^{1+\epsilon}}{\pi}\gamma\right)\right)
\right).
\end{equation}
For 
\begin{equation}\label{eq:mey_wave_03}
G(\gamma)
=
\begin{cases}
0,& \gamma<-1 \\
\frac{35\pi}{64}\left(\frac{-1}{7}\gamma^7+\frac{3}{5}\gamma^5-\gamma^3+\gamma+\frac{16}{35}\right),&
-1\leq \gamma <1\\
\frac{\pi}{2}, & \gamma\geq 1,
\end{cases}
\end{equation}
the resulting function $h_{\epsilon}$ and a dilated version are plotted in Figure
\ref{fig:mother_wave_01}.

%{We note that it is also possible to generate an overlapping wavelet scheme from the classical wavelets, however we stick to our construction due to the large non deforming zone it provides. }

% \begin{equation}
% h(\rho) = 
% \begin{cases}
% \sin\left(\frac{\pi}{2}\nu\left(\frac{4\rho}{\pi}-1\right)\right), &
% \frac{\pi}{4} < \rho \leq \frac{\pi}{2}\\
% \cos\left( \frac{\pi}{2}\nu\left(\frac{2\rho}{\pi}-1\right)\right), & \frac{\pi}{2} < \rho \leq \pi\\
%       0, & \mathrm{otherwise}\\
% \end{cases}
% \end{equation}
% For the Meyer wavelet of order $N$, the auxiliary function
%  $\nu(t)$ is a polynomial of degree $2N + 1$, which is chosen such that:
%  $\nu(t) = 0$,  if $t \leq 0$, $\nu(t) = 0$, if $t \geq 1$, and $\nu(t) + \nu
%  (1 - t) = 1$, while $\nu \in C^N([0, 1])$. E.g., the auxiliary function that
%  achieves a frequency response with $N = 3$ continuous derivatives is $\nu(t) =
%  t^4(35 - 84t + 70t^2 - 20t^3)$.
% 

\section{Localized Frequency}\label{sec:freq_loc}

In this section, we consider the wavelet frame defined by the Meyer-type mother
wavelet of Section 2 combined with the multipliers defined in Proposition
\ref{prop:trig_poly}.  We require the energy of the multipliers to be localized
within the support of the wavelets, so that scaling the multipliers will provide
a close approximation to scaling the wavelets.

\subsection{Detailed Example} 
\label{sec:freq_loc_example}

Here is an example of the proposed construction.
We define the radial mother wavelet
$\phi$ in the Fourier domain by
$\widehat{\phi}(\bm{\omega})=h_{\epsilon}(\rho)$ where $h_\epsilon$ is
given as in \eqref{eq:mey_wave_01}, \eqref{eq:mey_wave_02}, and
\eqref{eq:mey_wave_03}.
Fourier-domain representations of $h_{\epsilon}(\rho)$ and
$h_{\epsilon}(\rho/2)$ are shown in Figure \ref{fig:mother_wave_01}.
\begin{figure}[t!]
  \centering
  \includegraphics[width=6cm]{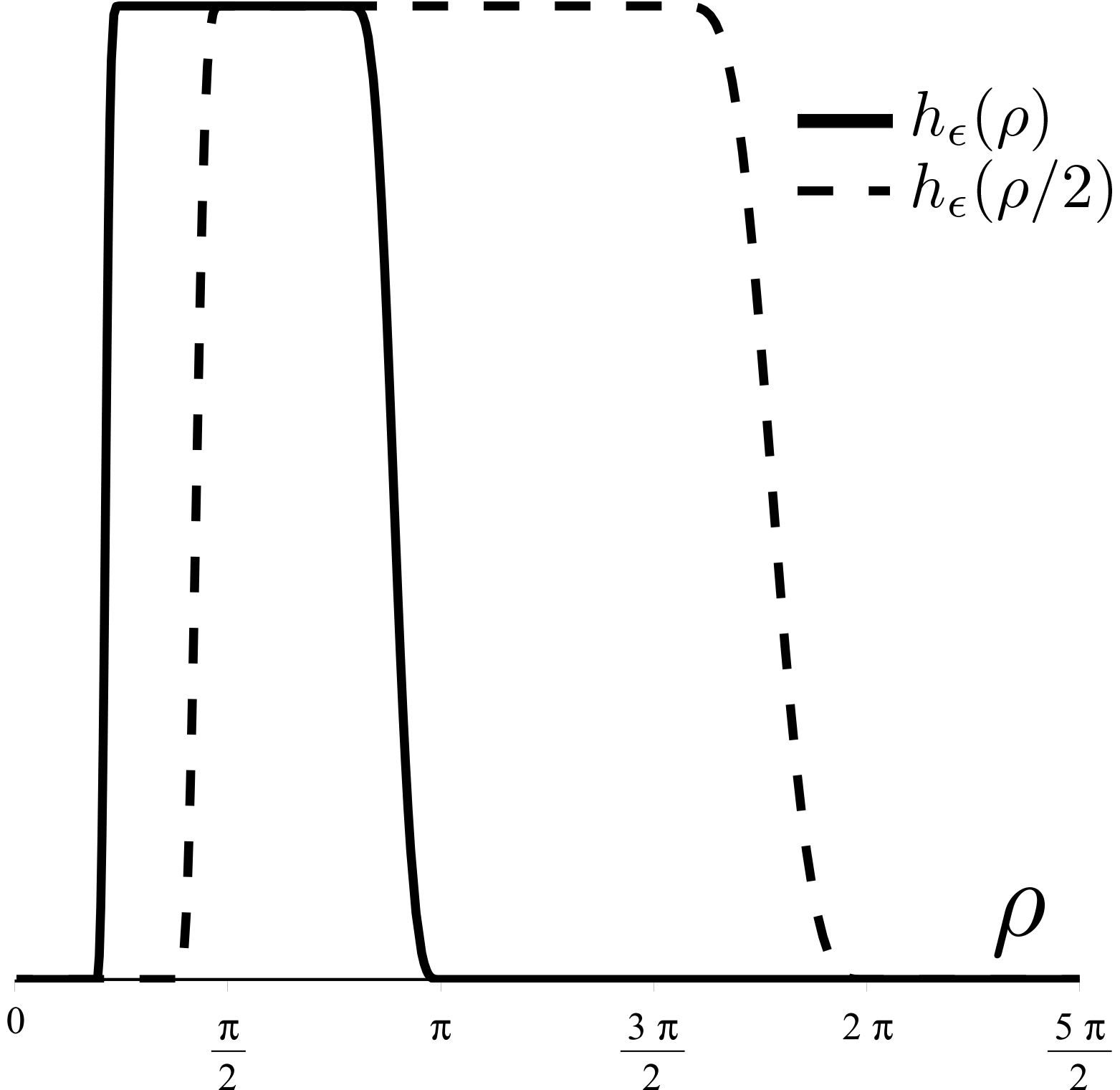}
  \caption[Mother wavelets]
   {Plot of $h_{\epsilon}(\rho)$ and $h_{\epsilon}(\rho/2)$ for the Meyer
   wavelet system.}
  \label{fig:mother_wave_01}
\end{figure}
We extend this wavelet frame using the admissible collection of Proposition
\ref{prop:trig_poly} with ${n_{\text{max}}}=2{l_{\text{max}}}+1$. 
The trigonometric polynomial $m$ is defined by the vector 
\begin{equation}
\bm{\alpha} 
=
\frac{\sqrt{4685}}{14055}
\left(
\begin{matrix}
125, & 101\sqrt{2},  & 53\sqrt{2}, & 16\sqrt{2}, & 2\sqrt{2}
\end{matrix}
\right),
\label{eq:alphaval}
\end{equation}
which is determined by sampling a polynomial B-spline of degree $3$. It is shown in Figure \ref{im:multiplier}.
We note that the choice of the multiplier depends on the application. The one used in this paper corresponds to a wavelet that has good localization in the Fourier domain. 
For $\bm{\alpha}$ defined in this way, we see that
the collection is admissible by Theorem
\ref{prop:trig_poly}, and the extended wavelet frame is tight.

\begin{figure}[t!]
  \centering
  \includegraphics[width=6cm]{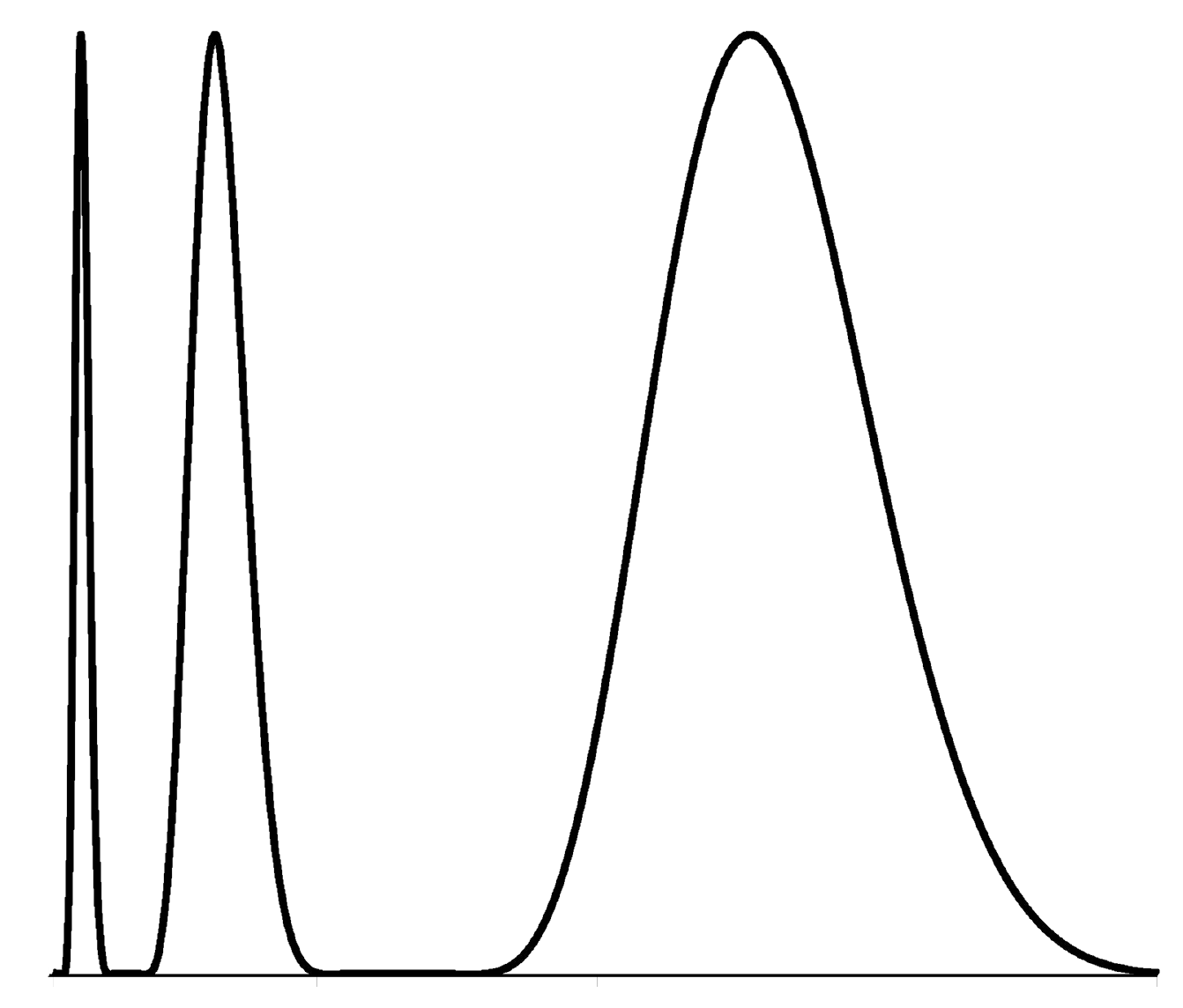}
  \caption[Multiplier]
   {Plot of the Fourier multiplier $m \left(\log_2(\cdot) \right)$ on the interval $[ \pi / 32, 2\pi ]$.}
   \label{im:multiplier}
\end{figure}

\subsection{Pseudo-Scaling} 

Our construction allows us to perform pseudo-dilations of the wavelets
within the extended tight frame of Section \ref{sec:freq_loc_example}.
To explain this procedure, we define the scaled multipliers
\begin{equation}
\mathcal{M}^a 
:=
\left\{m \left( \log_2 \left(\rho_n\abs{a\cdot} \right)
\right) \right\}_{n=1}^{n_{\text{max}}}. 
\end{equation}
We recall that $h_{\epsilon}$ is the Fourier profile of the primal mother wavelet $\phi$.
For a given parameter 
\begin{equation}
0<\epsilon'<\frac{\pi}{2}\left( 1-\frac{2}{4^{1+\epsilon}}  \right),
\end{equation}
we define the interval
\begin{equation}\label{eq:pseudo_scale_interval}
(4^{-1-\epsilon}\pi+\epsilon', 2(4^{-1-\epsilon}\pi+\epsilon')]
\end{equation}
within the support $[4^{-1-\epsilon}\pi,\pi]$
of the profile $h_{\epsilon}$. Furthermore, note that the $2^q$ dilations of this
interval are non-overlapping.

Here, we formally define the operation of pseudo-scaling wavelets as
scaling of the corresponding multipliers.  The pseudo-scaled wavelet system is
formed by applying $\mathcal{M}^a$ to the primal-wavelet system. Within this
system, we can very closely approximate true scaling. We consider the multiplier defined in Section \ref{sec:freq_loc_example}.  

Let $M_0$ be an element of $\mathcal{M}^1$ such that the profile of $\widehat{\phi}M_0$ attains
its maximal value in the interval \eqref{eq:pseudo_scale_interval}. Let
 $p_0$ denote the location of the maximum, and define $p(a)=p_0/a$, as illustrated in Figure
 \ref{fig:pseudo_dil}. We define the wavelet $\psi:=\mathcal{F}^{-1}\{ M_0
 \widehat{\phi} \}$. The profile of the dilated version of the Fourier transform of $\psi$, $\widehat{\psi} \left( a^{-1}\cdot \right)$, attains a maximum at $p(a)$.  The
pseudo-dilated version of $\psi$ is defined as
\begin{equation}
\psi_a := \mathcal{F}^{-1} \left\{ M_0(a\cdot)
\widehat{\phi} \left( 2^{q_a} \cdot \right) \right\},
\end{equation}
where $q_a$ satisfies 
\begin{equation}
p(a)\in \left(2^{q_a} ( 4^{-1-\epsilon}\pi+\epsilon' ), 2^{q_a} 2(4^{-1-\epsilon}\pi+\epsilon') \right].
\end{equation}
In other words, $p(a)$ belongs to the $2^{q_a}$ dilation of the interval
\eqref{eq:pseudo_scale_interval}.  The Fourier transform of $\psi$ and
a sequence of pseudo-dilations are shown in Figure
\ref{fig:pseudo_dil}. Note in particular the transition as $p(a)$
crosses the point $2(4^{-1-\epsilon}\pi+\epsilon')$. The wavelets maintain their shape due to the overlapping primal windows.

We measure the correlation between the truly scaled $\psi$ and its pseudo-dilated counterpart
$\psi_a$  as
\begin{equation}
\varrho_{\psi}(a)
:=
\max
\left(\left\langle\frac{\psi_a}{\norm{\psi_a}_{L_2}}
,\frac{\psi(a\cdot)}{\norm{\psi}_{L_2}}\right\rangle,0\right).
\label{eq:qm}
\end{equation}

\begin{figure}[t!] 
  \centering
  \includegraphics[width=6cm]{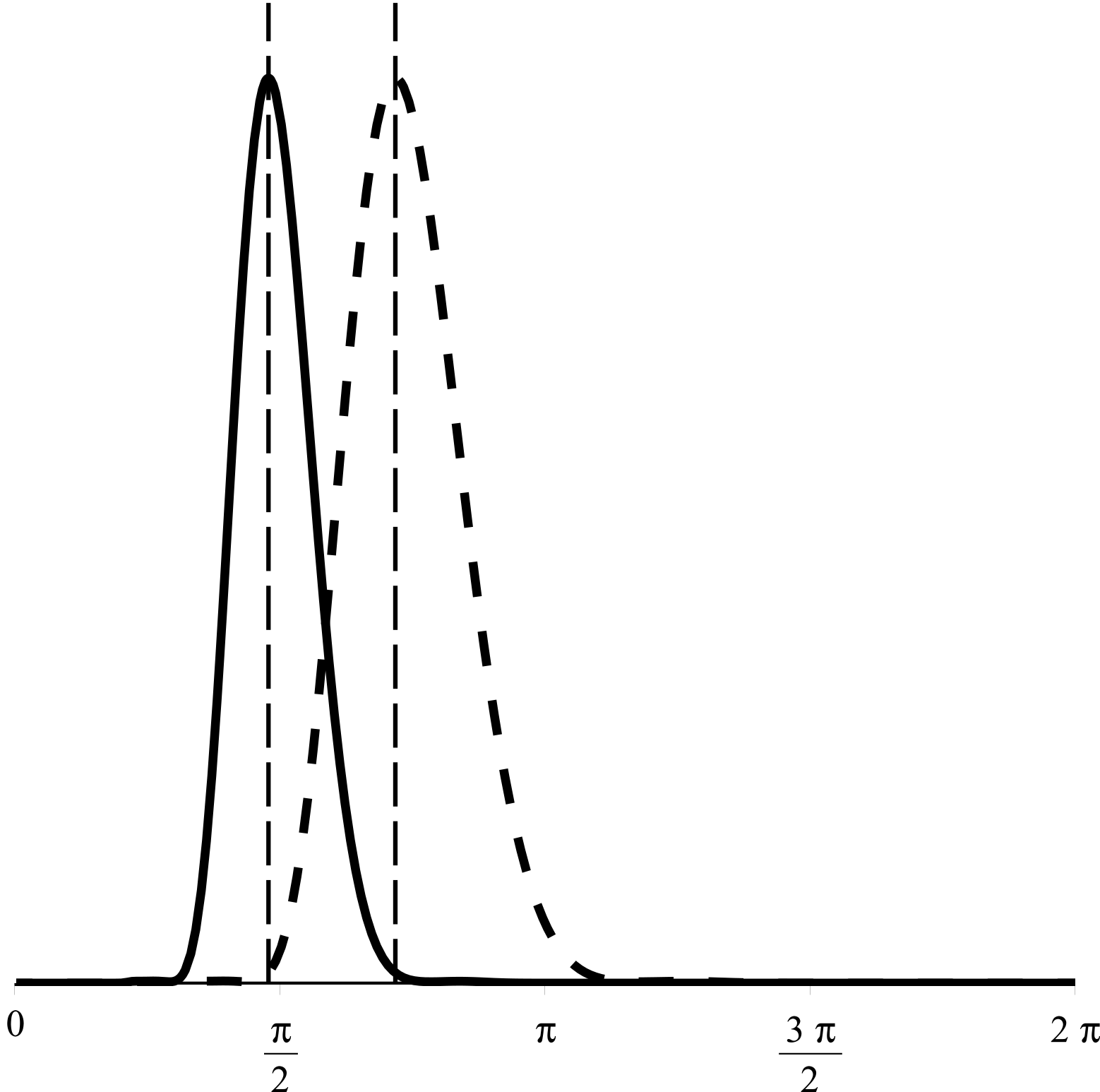} 
  \includegraphics[width=6cm]{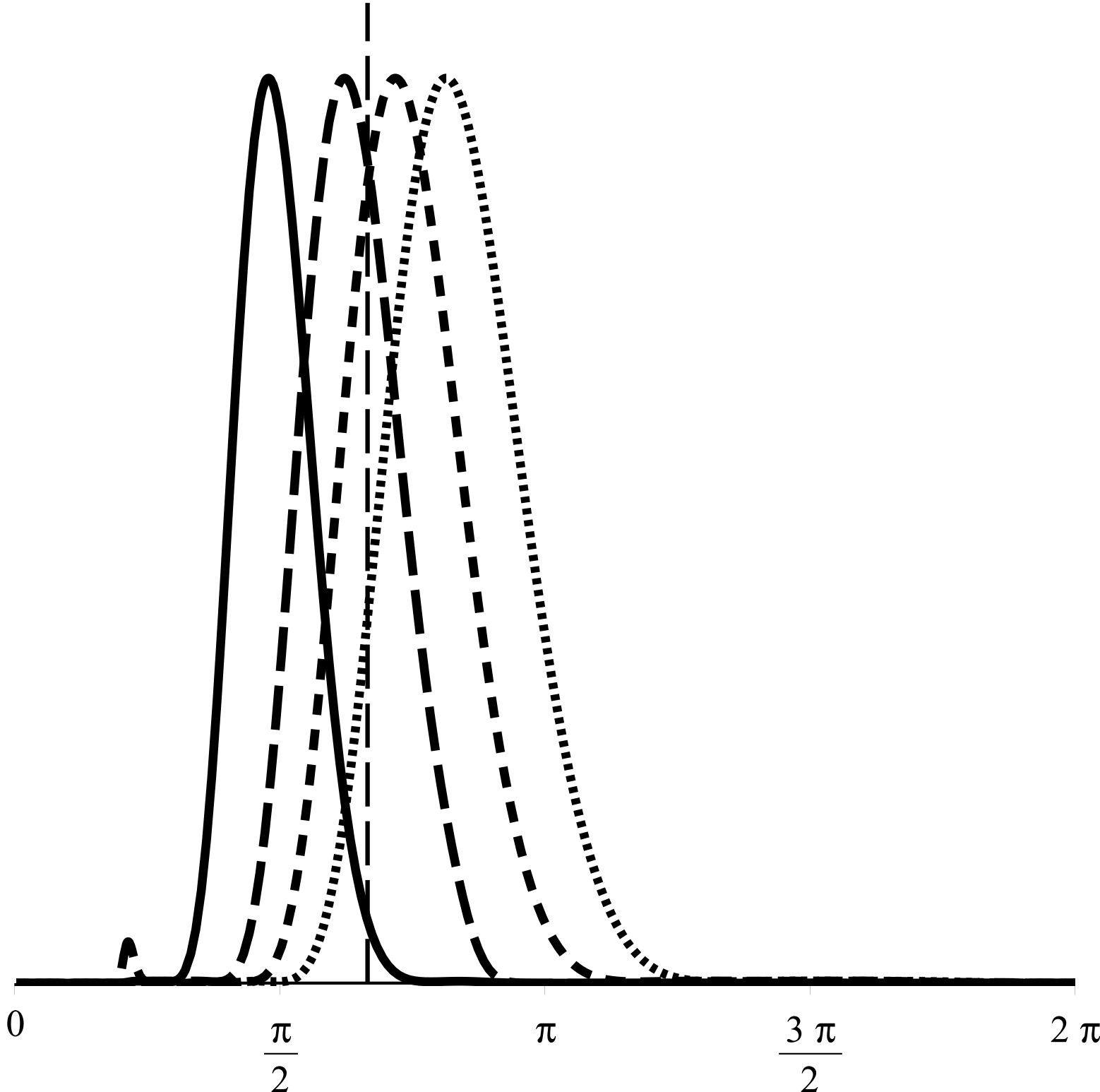} 
  \caption[Frame Element] 
   {Left: Profile plot of $\widehat{\psi}(a^{-1} \cdot)$ for $a=1.0$ (solid) and $a=1.5$
   (dotted). The location of the peak is $p(a)$, which is indicated by the dashed
   vertical line. Right: Profile plot of pseudo-dilations $\widehat{\psi}_a$, for
   $a=1,1.3,1.5,1.7$.  The dashed vertical line is
   located at $2(4^{-1-\epsilon}\pi+\epsilon')$.
   }
   \label{fig:pseudo_dil}
\end{figure}

 \begin{figure}[!t]
\centering
	\subfloat{\includegraphics[width=.68\textwidth]{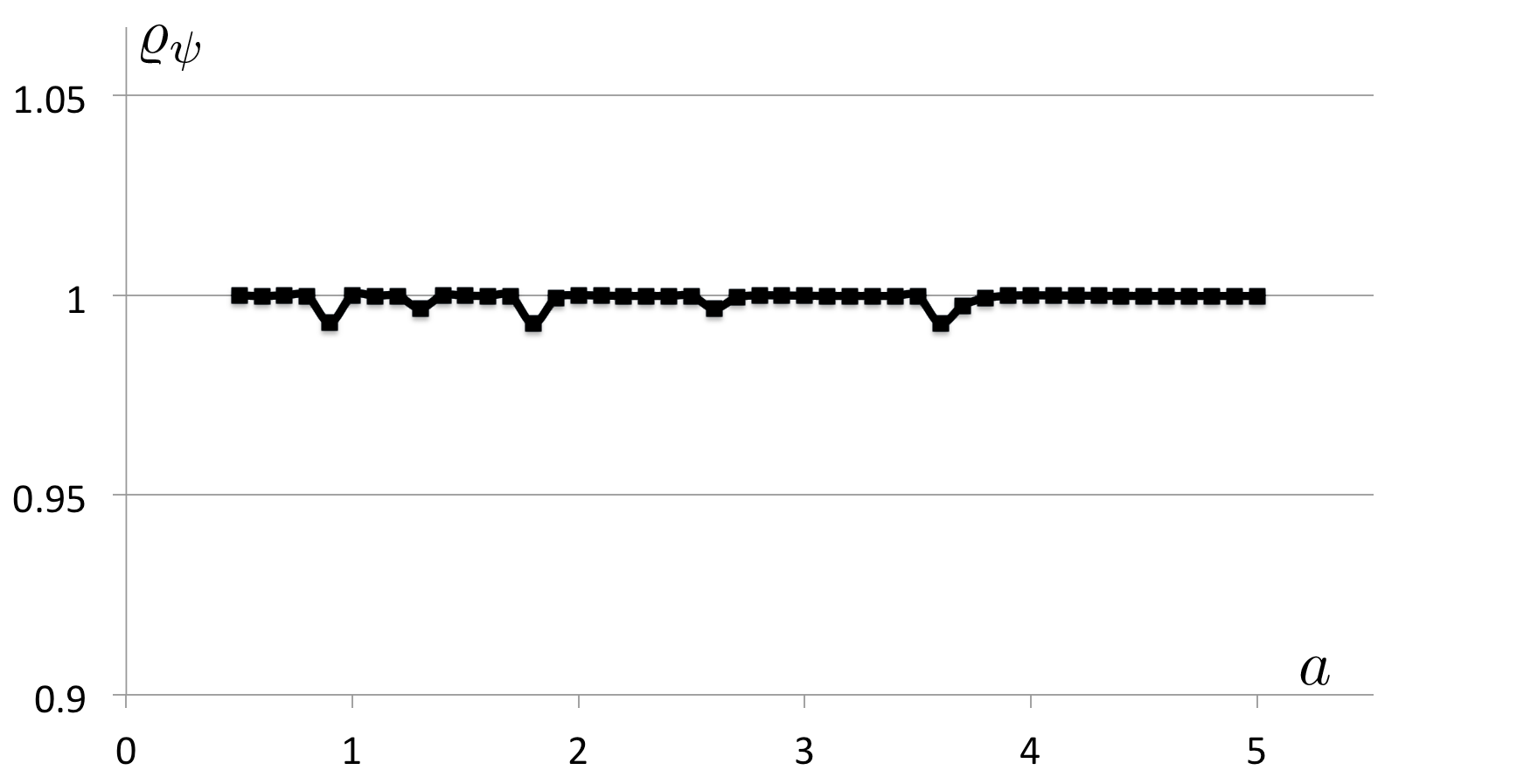}}\\
		\caption{Quality metric $\varrho_{\psi}$ as a function of the dilation $a$.}	
 	\label{im:quality}	
\end{figure}

% Note that pseudo-scaling
% operation closely resembles a true scaling of the wavelet function, as the
% largest portion of the multiplier stays away from the boundary of the support
% of the primal wavelets. 
This is similar to the
quality metric defined in \cite{r:chaudhury} to measure the approximation of
shiftable wavelets. Values close to one indicate that the pseudo-dilated
wavelets provide a close match to true dilation, while smaller values indicate a
poor match. We show in Figure \ref{im:quality} the evolution of $\varrho_{\psi}$ as a function of dilation.
There, we illustrate the quality of pseudo-scaling using the empirically defined value 
$\epsilon'=0.45$.  Due to the periodic
nature of pseudo-scaling, the minimum of $\varrho_{\psi}$ is found by analyzing a single period
of the trigonometric function. We compute this minimum to be 
$0.998$.

\subsection{Comments on the Construction}

In Proposition \ref{prop:trig_poly}, the periodicity of the trigonometric functions means
that the multipliers are periodic in scale.  For example,
\begin{equation}
\cos \left( \pi l \log_2(2^2 \cdot)\right) 
= 
\cos \left( \pi l \log_2( \cdot)\right),
\end{equation}
since $l$ is an integer.

Essentially, our tight-frame construction consists of the application of an
admissible collection $\{M_n\}$ to bandlimited wavelets $\phi$, where
$\widehat{\phi}$ is zero in a neighborhood of the origin. Therefore, the
supports of the scaled wavelets $\widehat{\phi}(a^q\cdot)$ segment the domain so that, 
for each scale, the support of $\widehat{\phi}(a^q\cdot)$ will cover a
finite number of scaling periods (of the form $[\rho,r\rho]$) of the
multipliers.
The fact that cyclic scaling can be combined naturally with a geometric
division of the Fourier domain motivates our interest in cyclically scalable
families for use with wavelets.  Also note that there should exist
synchronization between the cyclic periodicity and the discrete scales of the
wavelets, the usefulness being that one could get the rough discrete scale from
the wavelet hierarchy, the local scale being refined by
adapting the multipliers. Furthermore, we can scale the frame independently at each point, due
to the transformability property seen in Proposition \ref{prop:sc_form}.

\section{Special Case: One Complex Multiplier} 
\label{sec:two_channel}

In the localized-frequency construction, arbitrary localization can be
obtained if one is willing to work with a highly redundant wavelet frame. 
Here, to introduce the ideas, we provide details for the simplified case where the admissible
collection contains only two functions: sine and cosine. 

We propose to use the mother wavelets $\psi^{\cos}$ and $\psi^{\sin}$ that are defined by their Fourier domain profiles
\begin{align}
\widehat{\psi}^{\cos}( \bm{\omega}) 
&=  
h(| \bm{\omega} | )
\cos(\omega_0\log_2(\kappa | \bm{\omega} | ))\\
\widehat{\psi}^{\sin}( \bm{x}) 
&= 
h(| \bm{\omega} |)
\sin(\omega_0\log_2(\kappa| \bm{\omega} |)),
\end{align}
which have two parameters: $\kappa$ and $\omega_0$. The parameter $\kappa$
specifies a local shift of the cosine and sine under the window
$h$ in the frequency domain. The dependence of this frequency shift on
$\kappa$ is cyclic: all values of $\kappa$ that differ by $2^{n 2\pi/\omega_0}$
produce equivalent frequency shifts. The second parameter is $\omega_0$. It
determines the size of these local-frequency cycles as well as the number of
oscillations of cosine and sine that fall within the support of $h$.
In practical designs, one can tune these parameters to select the number of
oscillations ($\omega_0$) and the phase shift ($\kappa$) of the sinusoids within
the support of $h$. In \cite{ISBI2015}, the authors were using $\omega_0 = 4\pi$ and $\kappa=2^5 / \pi$. The function $h$ can be defined as
$2^{-1/2}h_{\epsilon}$, where $h_{\epsilon}$ is from
\eqref{eq:mey_wave_02}.

The two real mother wavelets are combined in the single complex wavelet
\begin{align}
	\psi(\bm{x}) 
	&= 
	\psi^{\cos}(\bm{x}) + \uj \psi^{\sin}(\bm{x}) \\
	&= 
	\mathcal{F}^{-1}\left\{h(|\cdot|)
	\ue^{\uj \omega_0\log_2(\kappa |\cdot|)}\right\}(\bm{x}) .
\end{align}
The polar form of the analysis coefficients of the complex wavelets $\psi_{s,\V
k}$ is 
\begin{equation}\label{eq:polar_wave_coef}
A_{s, \bm{k}}\ue^{\uj \beta_{s, \bm{k}}} 
= 
\left\langle f, \psi_{s, \bm{k}}
\right\rangle,
\end{equation}
where $s\in\mathbb{Z}$ denotes the scale and $\bm{k}\in\mathbb{Z}^d$ the
position. 

We note the following scaling
relationship between $\psi$ and $\psi^{\cos}$:
\begin{align}
\Re\left(\ue^{\uj\beta} \psi\right) 
&= 
\cos(\beta)\psi^{\cos} -
\sin(\beta)\psi^{\sin}\\
&= 
\mathcal{F}^{-1}\left\{  h(|\cdot|) \cos(\omega_0
\log_2(2^{\beta/\omega_0}\kappa |\cdot|)) \right\}.
\end{align}
A similar relationship holds between $\psi$, $\psi^{\cos}$, and $\psi^{\sin}$ upon
taking the imaginary part. When analyzing a real-valued function $f$, we use
this property to scale the complex wavelets. We arrive at an expansion of $f$ as
a sum of locally scaled versions of $\psi^{\cos}$, where $\psi^{\cos}$ is adapted to the signal at each scale $s$
and position $\bm{k}$ as in 
\begin{align}
f 
& =
\sum_{s, \bm{k}} \left\langle f, \psi_{s, \bm{k}}
\right\rangle \psi_{s, \bm{k}} \\
& = 
\sum_{s, \bm{k}} \left( A_{s, \bm{k}}  \ue^{\uj\beta_{s, \bm{k}}}
\right) \psi_{s, \bm{k}} \\
& =
 \sum_{s, \bm{k}} A_{s, \bm{k}} \left(\ue^{\uj\beta_{s, \bm{k}}}
\psi_{s, \bm{k}} \right) \\
& = 
\sum_{s, \bm{k}} A_{s, \bm{k}}
\mathcal{F}^{-1}
\left\{ 2^{sd}\ue^{-\uj 2^s\bm{k}\cdot\bm{\omega}} 
h(2^s |\cdot |) 
\cos
\left(\omega_0 \log_2 \left(2^{\beta_{s,k}/\omega_0}\kappa |\cdot| \right)\right) \right\}.
\end{align}

\begin{definition}
Let $\left\{\psi_{s,\bm{k}}\right\}$ be a wavelet system as defined above. For a
function $f:\mathbb{R}^d\rightarrow \mathbb{R}$, let the wavelet coefficients be
defined as in \eqref{eq:polar_wave_coef}. Then, we define the scale-adapted
wavelets $\psi_{s,\bm{k}}^f$ by
\begin{align}
\psi_{s,\bm{k}}^f 
&=  
\mathcal{F}^{-1}\left\{2^{s d}\ue^{-\uj2^s\bm{k}\cdot\bm{\omega}}  
h(2^s |\cdot|) 
\cos
\left(\omega_0 \log_2 \left(\kappa_{s,\bm{k}} |\cdot| \right)\right) \right\} \\
\kappa_{s,\bm{k}}
&=
2^{\beta_{s,\bm{k}}/\omega_0}\kappa,
\end{align}
for $0\leq \beta_{s,\bm{k}}<2\pi$.
\end{definition}

Using the scale-adapted wavelets, we view $A_{s,\bm{k}}$ as the coefficient associated to the
analysis of $f$ with a version of $\psi^{\cos}$ that is optimally scaled locally.
Scaling $\psi_{s,\bm{k}}$ by $2^{(-\beta_{s,\bm{k}})/\omega_0}$
creates the best match between $f$ and a locally scaled version of $\psi^{\cos}$.

 \begin{figure}[!t]
\centering
	\subfloat{\includegraphics[width=.98\textwidth]{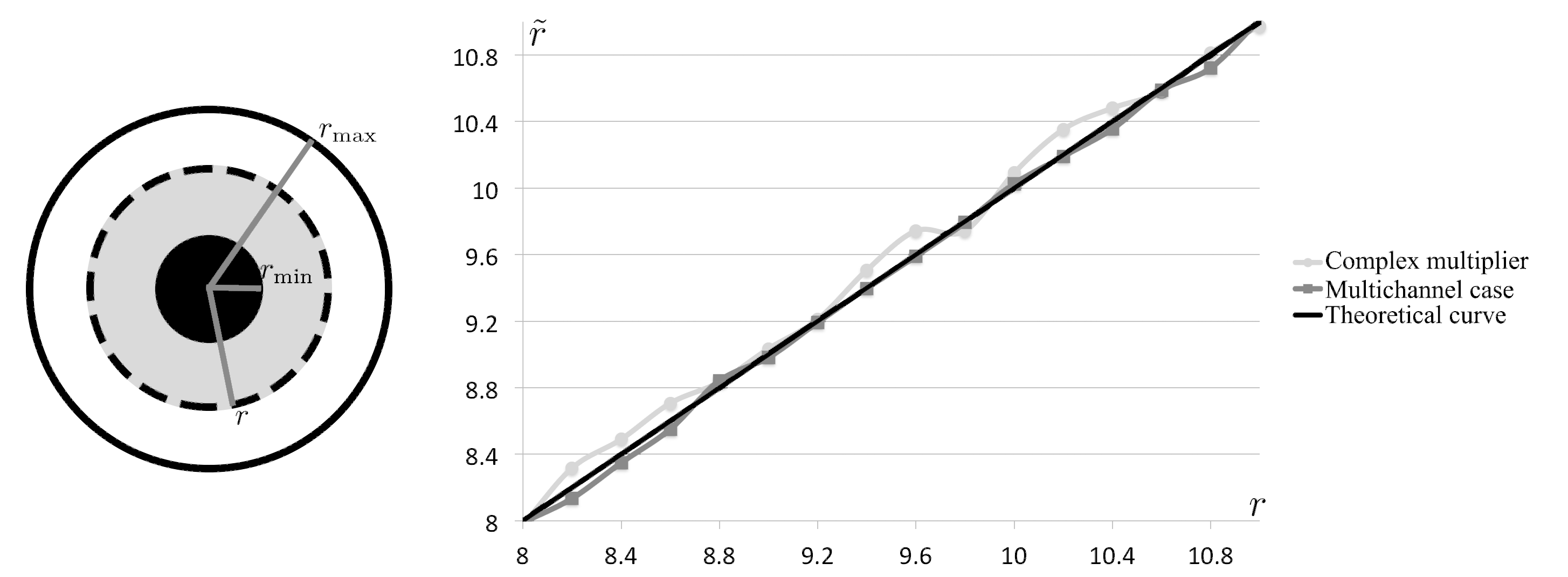}}\\
		\caption{Reference images (left) and radius estimations (right). The measured spot size as a function of the actual spot size. The straight black curve represents the theoretical values. The light grey curve shows the estimated radius extracted from the phase of the wavelet coefficients in the complex case. The dark grey curve visualizes the estimated radius determined by scaling the multipliers in the multichannel case.}	
 	\label{im:curve}	
\end{figure}

%The case of the one complex multiplier presented here has been already published in \cite{ISBI2015}. 
An application of the single complex multiplier case was presented in \cite{ISBI2015}.
To compare the estimation properties of the one complex multiplier with the more general multichannel case (with $\alpha$ corresponding to (\ref{eq:alphaval})), we generated a series of test images. The spots range in size from 8 pixels to 11 pixels with step size 0.2. The reference images (left) and the correspondence between the measured spot radius and the actual spot size is illustrated in Figure \ref{im:curve}. Since the multichannel case provides a better estimate of size (with a relatively small computational overhead), from now on we focus on that case.

\section{Spot Detection}
\label{sec:application}

We summarize the algorithm for the detection and size estimation of circular patterns that we call spots. We suppose that a scalable wavelet with the proper Fourier multipliers are at our disposal. 
\\
(1)  ({\bfseries Wavelet Analysis With Scalable Wavelets})   \\
We decompose the image with the scalable wavelet. At each location $\V k$ and dyadic scale $s$, we have $n_{\text{max}}$ channels, corresponding to the number of Fourier multipliers. The first stage outputs a map of wavelet coefficients $\{ w_n(s,\V k) \}_{n = 1}^{n_{\text{max}}}$.
\\
(2)  ({\bfseries Thresholding [Optional]  and Non-Maximum Suppression})   \\
We assume that the keypoints (center of spots) are sparse in the image. We thus ignore points where the response of the detector is small, by 
applying global threshold based on the $\ell_2$ norm of the wavelet coefficients $\{ w_n(s,\V k) \}_{n = 1}^{n_{\text{max}}}$ over the channels at each scale and location. We additionally apply local non-maximum suppression to prevent multiple detections of the same object.
\\
(3)  ({\bfseries Adaptation})   \\
We perform the adaptation step (steering of the scale) at each location that was retained by the previous step. The size of an object corresponds to a maxima in the response of the wavelet detector. The wavelets therefore have to be ``scaled'' to look for the precise scale that elicits the largest response. 
\\
(4)  ({\bfseries Postprocessing [Optional] })   \\
We rank the candidates based on a particular measure (e.g., strength of the response of the wavelet, contrast with respect to the neighboring background, SNR). We choose the best corresponding results.

\section{Experimental Results}
\label{sec:experiments}
Our algorithm to detect spots and measure their size has been programmed as a plug-in for the open-source image-processing software ImageJ \cite{Abramoff2004}. 
To evaluate the performance of the algorithm, we use a variety of test images, including synthetic ones and real micrographs.
The aim of our experiments is to measure the speed of our method, its accuracy, and its robustness against background signal.
We also want to compare our method to other popular spot-detection methods in the literature.

In the evaluation phase we use the Hungarian algorithm to match the detections with the nodes of the original grid. The detections are accepted if they are no further than 5 pixel from the original nodes. Otherwise, they are counted as false positives. To make a quantitative evaluation, we compute the Jacquard index and the root-mean-square error (RMSE) for the estimation of the position and radius. We note that the RMSE is computed for the matched detections only. 

\subsection{Synthetic Data}

There exist several approaches dedicated to the detection of circular objects and measuring their radii. We can separate these approaches into three main categories: classical global
methods (e.g., morphology and adaptive thresholding); methods based on the detection and analysis of edges and gradients (e.g., the circular Hough transform \cite{Hough1962} and active contours); and approaches based on filtering (e.g., Laplacian of Gaussian (LoG), determinant of Hessian (DoH), and wavelet-based techniques \cite{OlivoMarin20021989}). 

Global methods are mostly used for the evaluation of clear structures without background signal or noise. Their accuracy drops significantly for complex structures that appear in biological experiments, due to their sensitivity to noise. Edge-based methods are highly demanding in computational time and capacity. Filter-based methods include detectors with parametric templates that correspond to a specific range of sizes or scales. 

In \cite{Sage}, it was shown that the LoG filter can be likened to a whitened matched filter and offers optimal properties for detection in a broad category of images. In \cite{Lindeberg}, Lindeberg proposed a multiscale extension of the LoG-based detection scheme to overcome the limitation of the single scale. His method uses numerous passes of the LoG filter to capture the location and the size of spots. The computational efficiency of the algorithm highly depends on the diversity of the spot size. 

For the comparison, we have chosen the following methods: Multiscale Laplacian of Gaussian, circular Hough transform, and thresholding and connected components. 

\subsubsection{Speed}

\begin{figure}[!t] 
\centering
	\subfloat{\includegraphics[width=.58\textwidth]{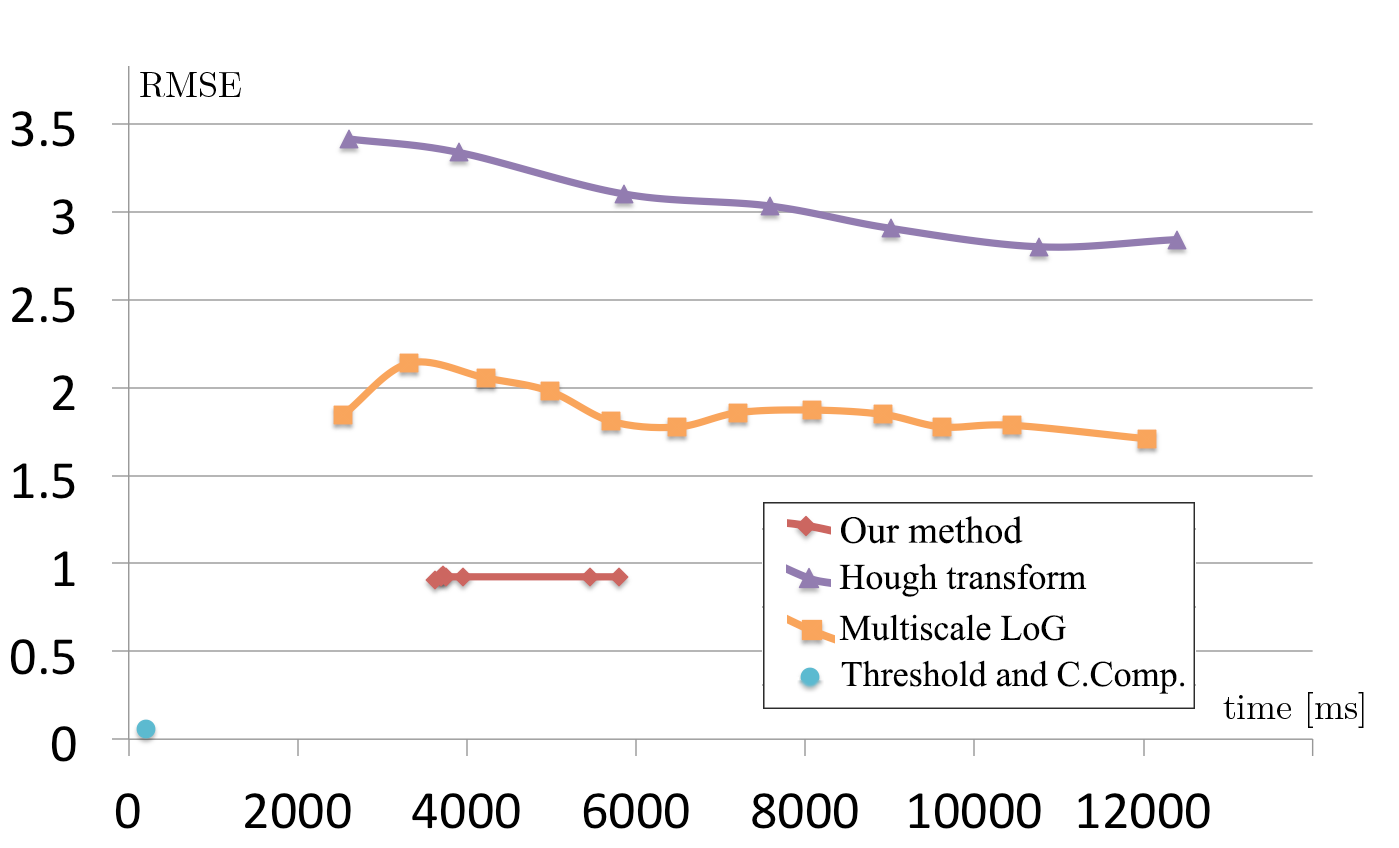}}
	\\
	\subfloat{\includegraphics[width=.58\textwidth]{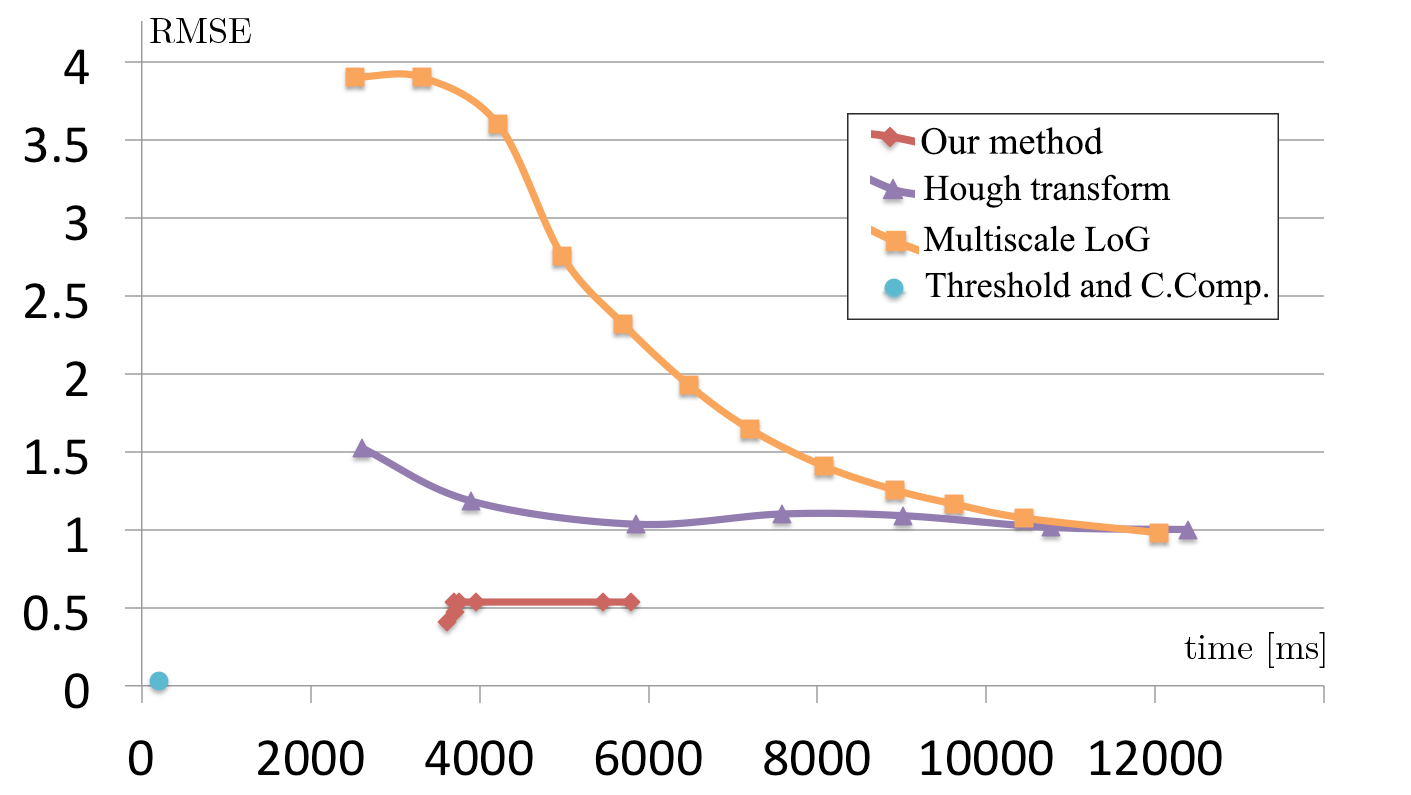}}
	\caption{Position and radius estimation error (in pixels) in a sense of root mean square error (RMSE) as a function of running time.}
	\label{im:speed}	
\end{figure}

First, we generate a series of test images (of size $\left(1,\!000 \times 1,\!000 \right)$) where
we control the location of the spots and their radius. The ground-truth data contain 200 disks, with radii varying between 8 and 40 pixels. We allow overlap between neighboring spots, by at most 10 pixels.

In the case of the multiscale Laplacian of Gaussian and the circular Hough transform, the parameters determine the precision of the algorithm (corresponding to their ``search space''). With generous settings they provide better results; however, their computation time increases dramatically. In the case of our method, we control the speed by modifying the global threshold for the wavelet coefficients (Step 2). The RMS errors of the estimation of position and radius as a function of running time are summarized in Figure \ref{im:speed}.  
We note that the application area of the method ``thresholding and connected components'' is limited to noise-free data. Based on the graphs, we conclude that our algorithm performs better for a given computation time than the competing methods. 

\subsubsection{Robustness Against Background Signal}
 
 \begin{figure}[!t]
\centering
	\subfloat{\includegraphics[width=.31\textwidth]{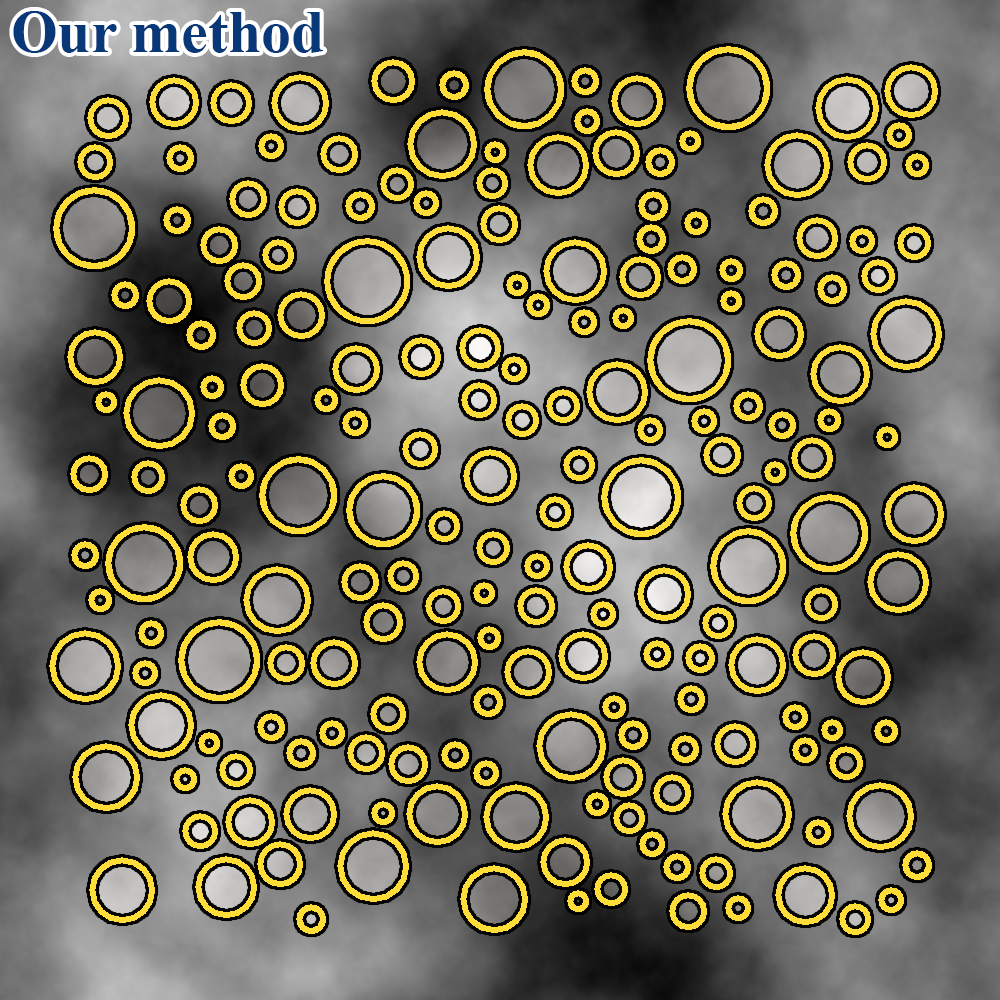}}
	\quad
	\subfloat{\includegraphics[width=.31\textwidth]{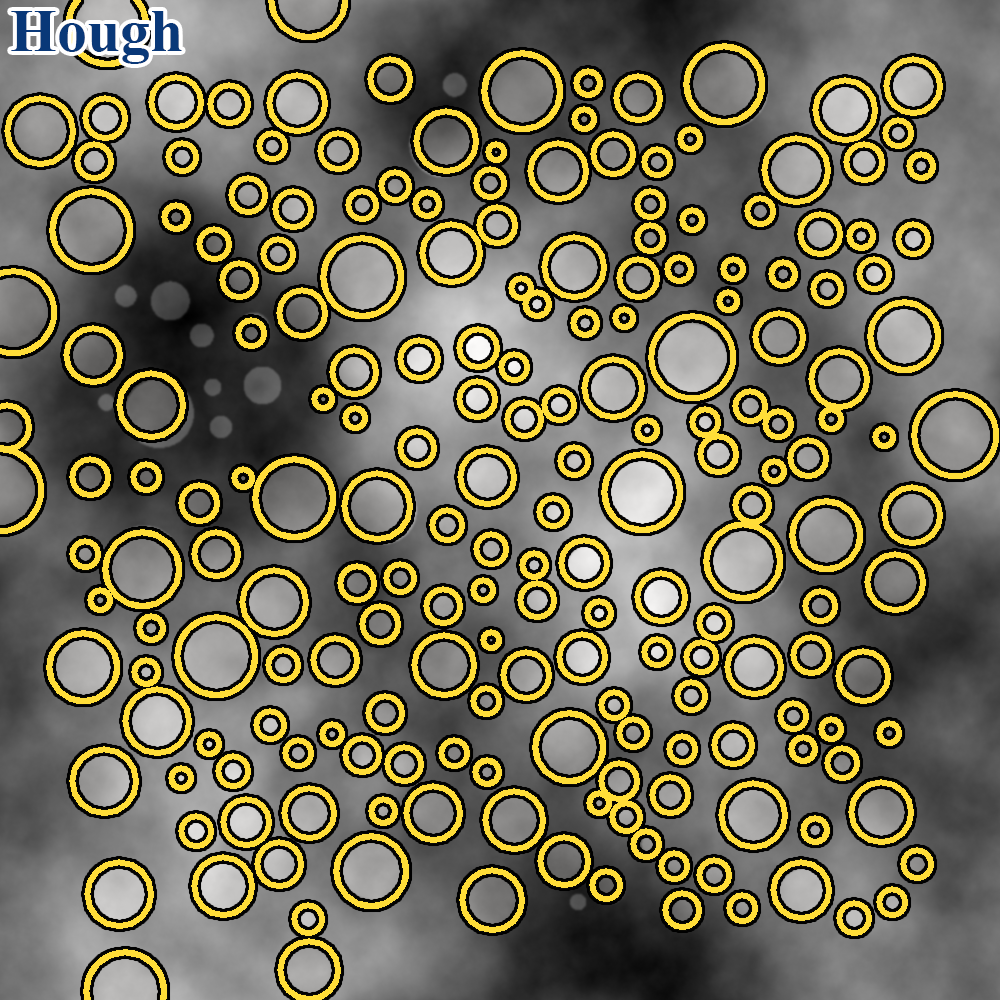}}
	\quad
	\subfloat{\includegraphics[width=.31\textwidth]{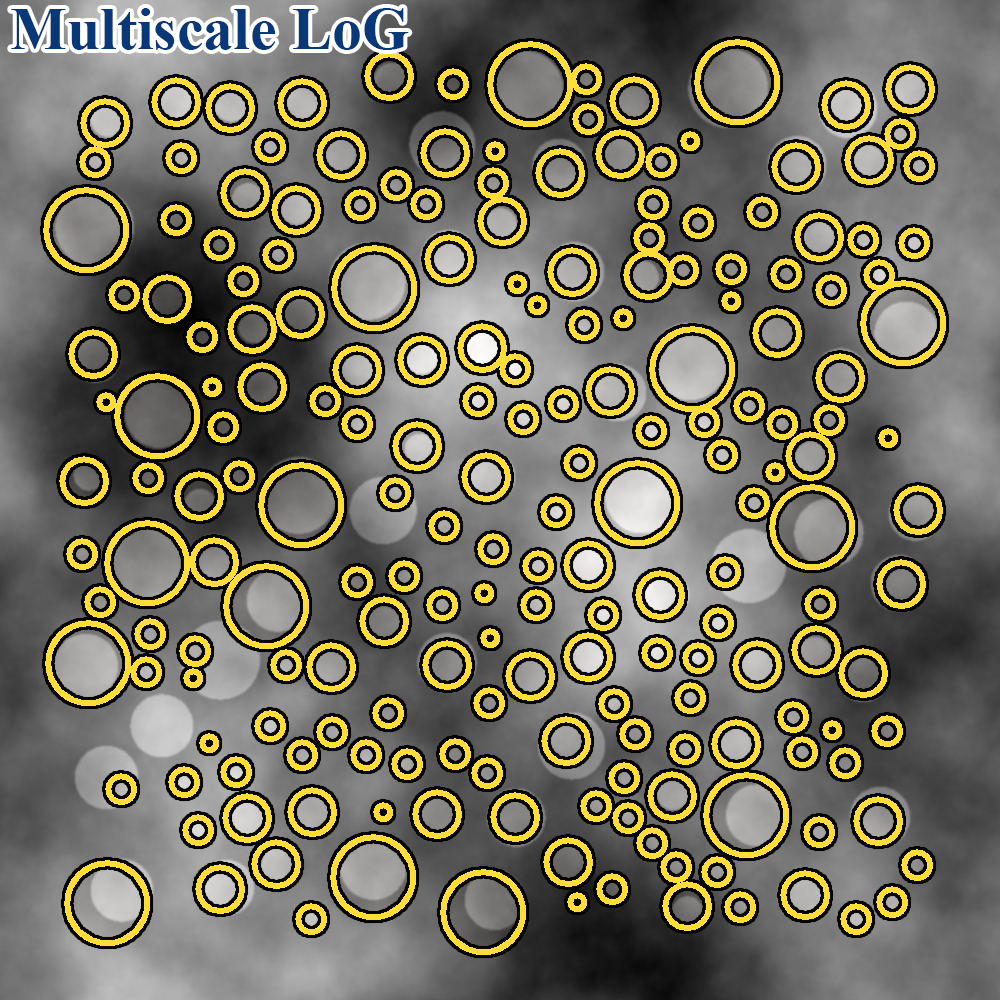}}\\
	\subfloat{\includegraphics[width=.31\textwidth]{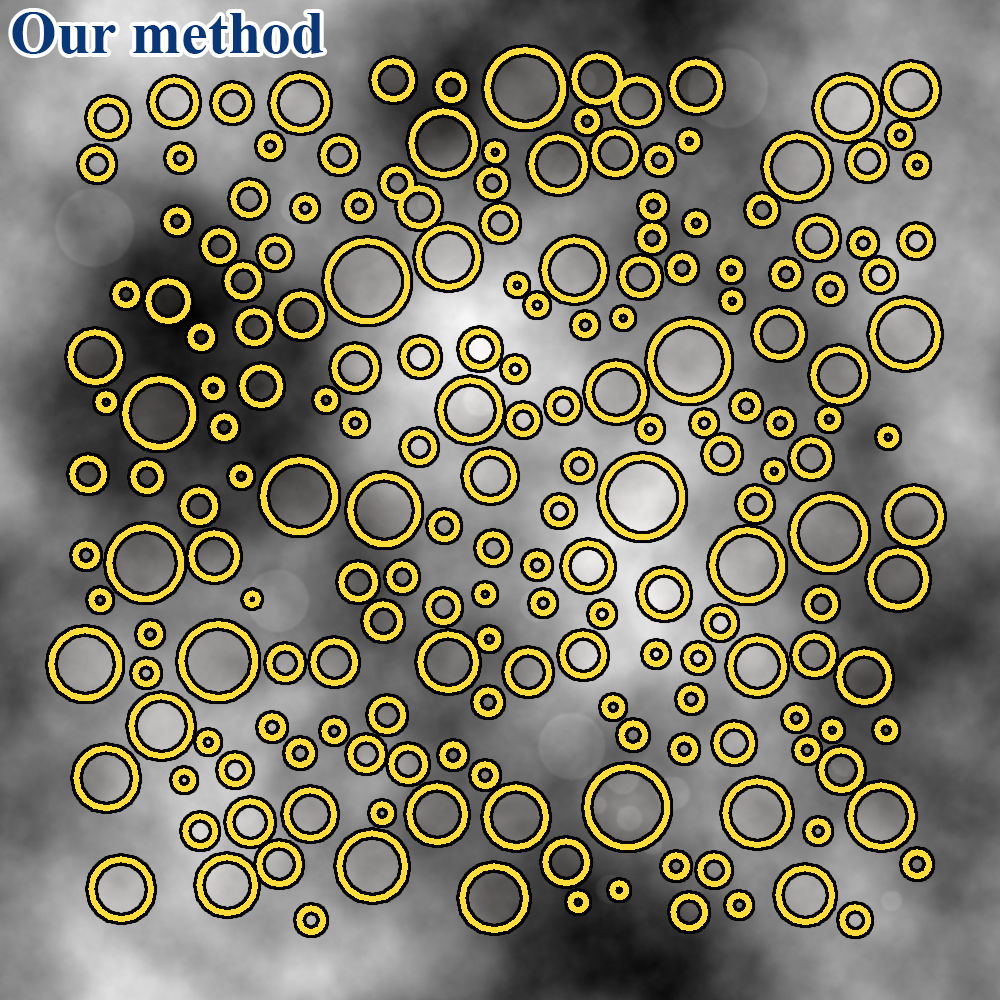}}
	\quad
	\subfloat{\includegraphics[width=.31\textwidth]{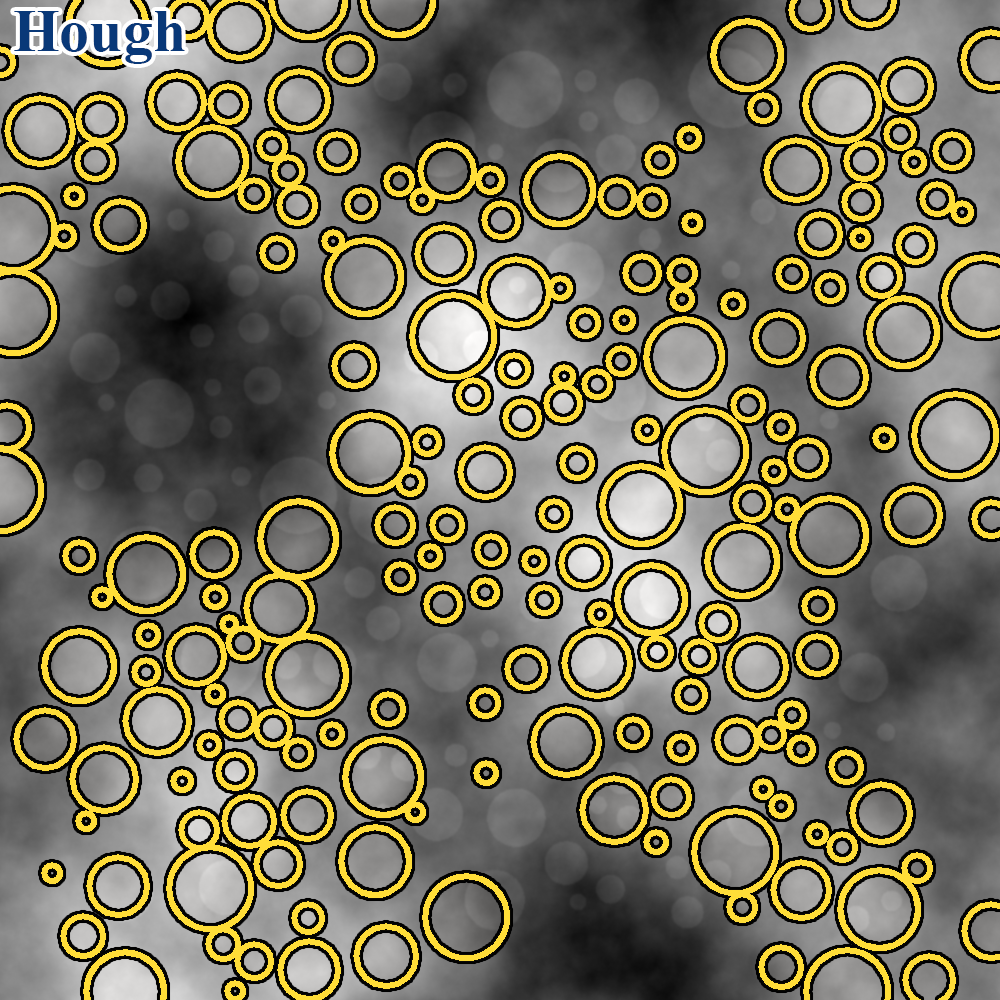}}
	\quad
	\subfloat{\includegraphics[width=.31\textwidth]{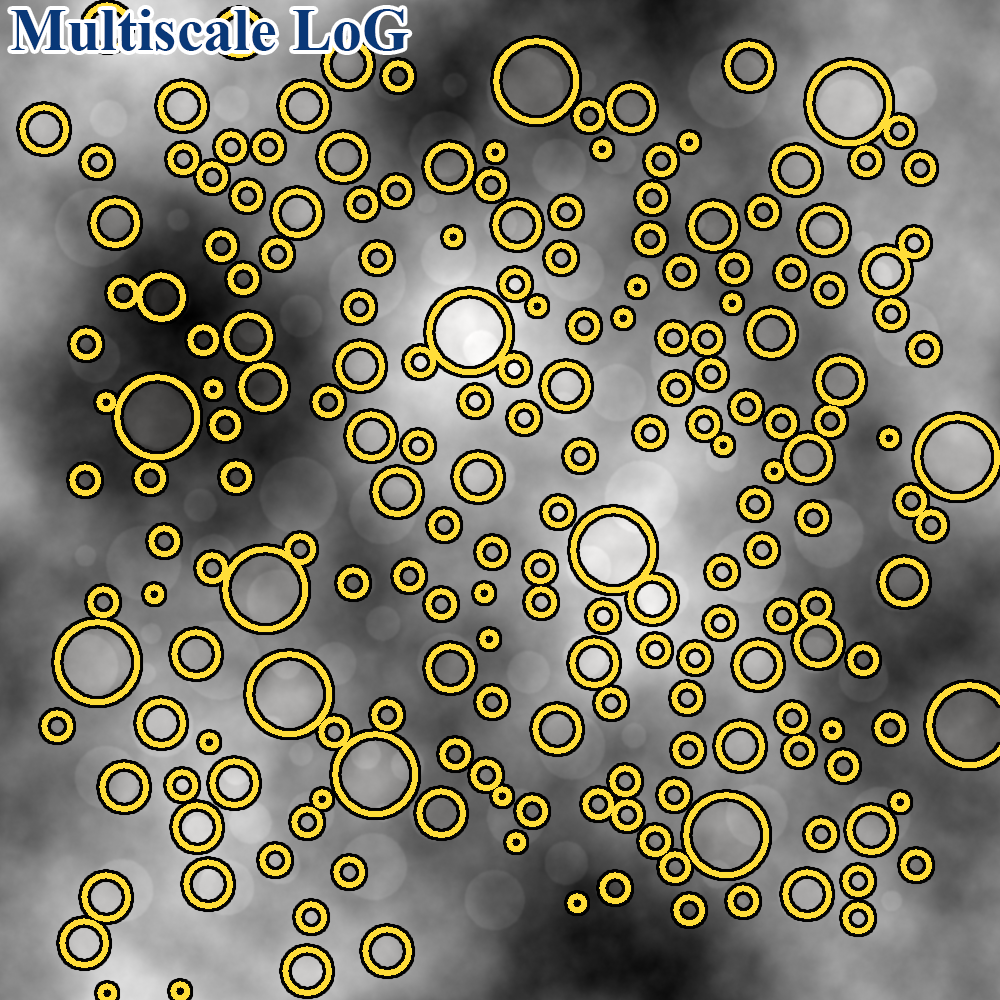}}\\
	\caption{From left to right and from top to bottom: detections on the test image corrupted with additive
isotropic Brownian motion (mean 0, standard deviation 2.0);  detections on the test image corrupted with additive isotropic Brownian motion (mean 0,
standard deviation 8.0).}	
	\label{im:series}	
\end{figure}
 \begin{figure}[!t]
\centering
	\subfloat{\includegraphics[width=.75\textwidth]{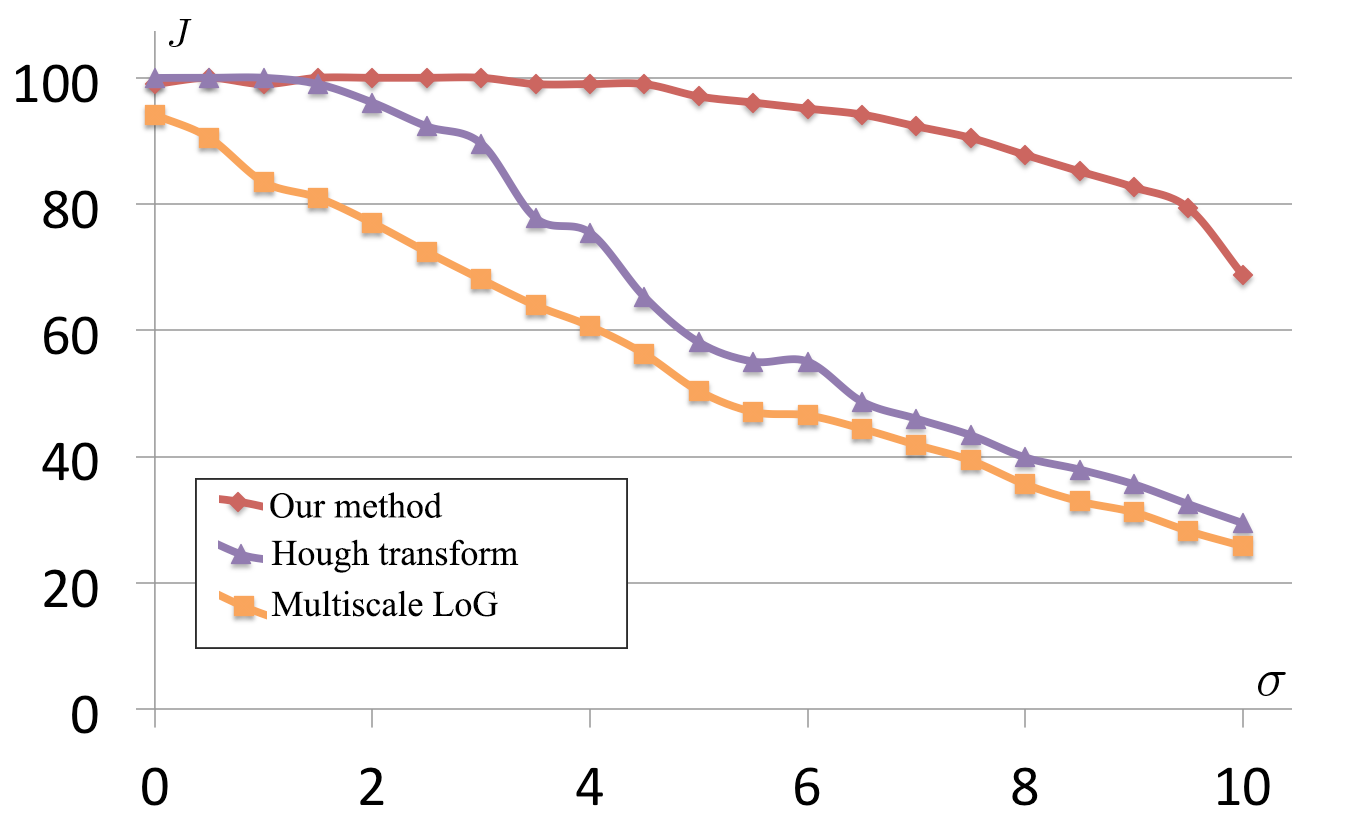}}\\
		\caption{Jaccard index under  isotropic Brownian
	motion (background signal), as a function of the standard
	deviation $\sigma$ of the noise.}	
 	\label{im:JaccardCloud}	
\end{figure}
 
\begin{figure}[!t]
\centering
	\subfloat{\includegraphics[width=.68\textwidth]{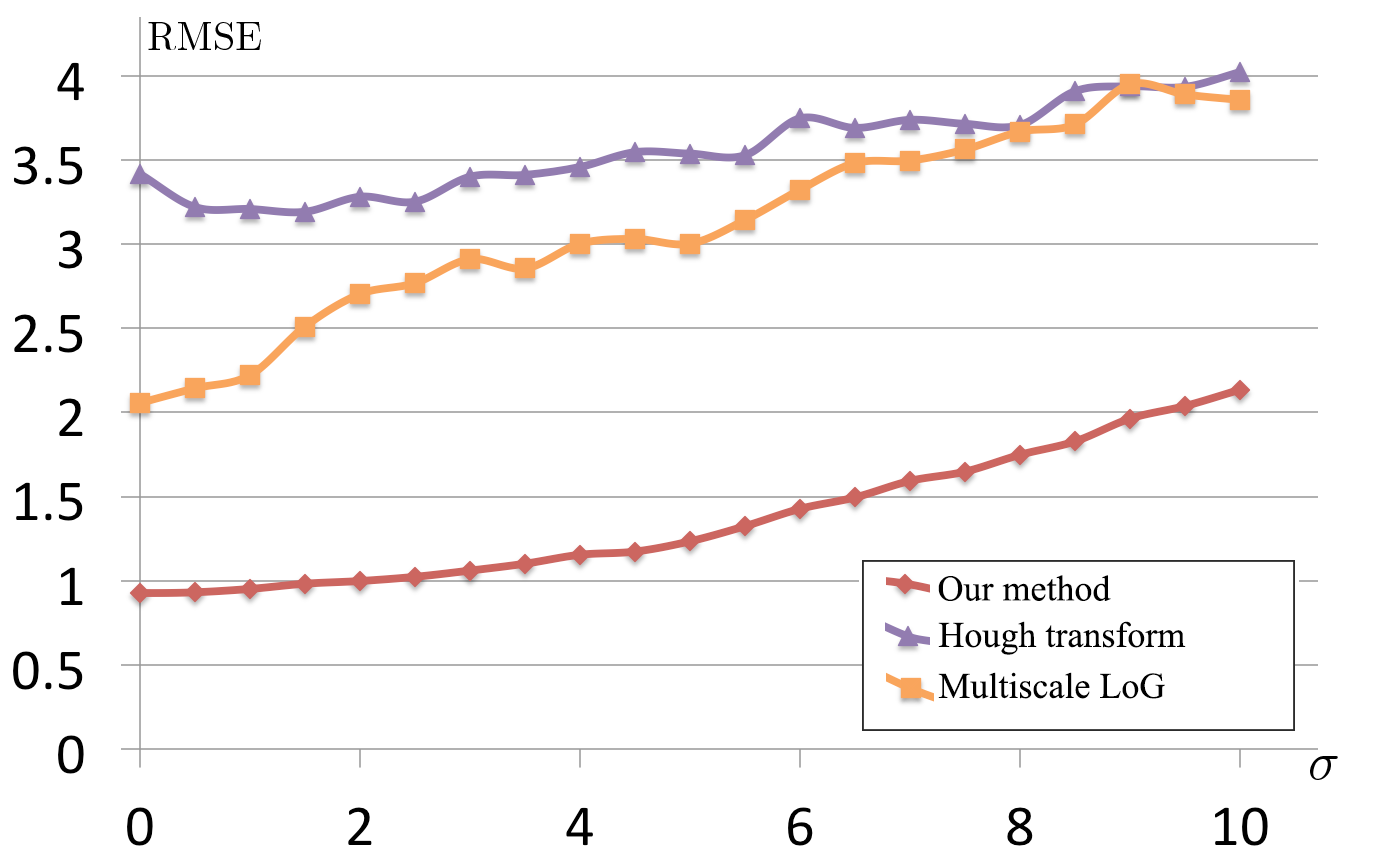}}
	\\
	\subfloat{\includegraphics[width=.68\textwidth]{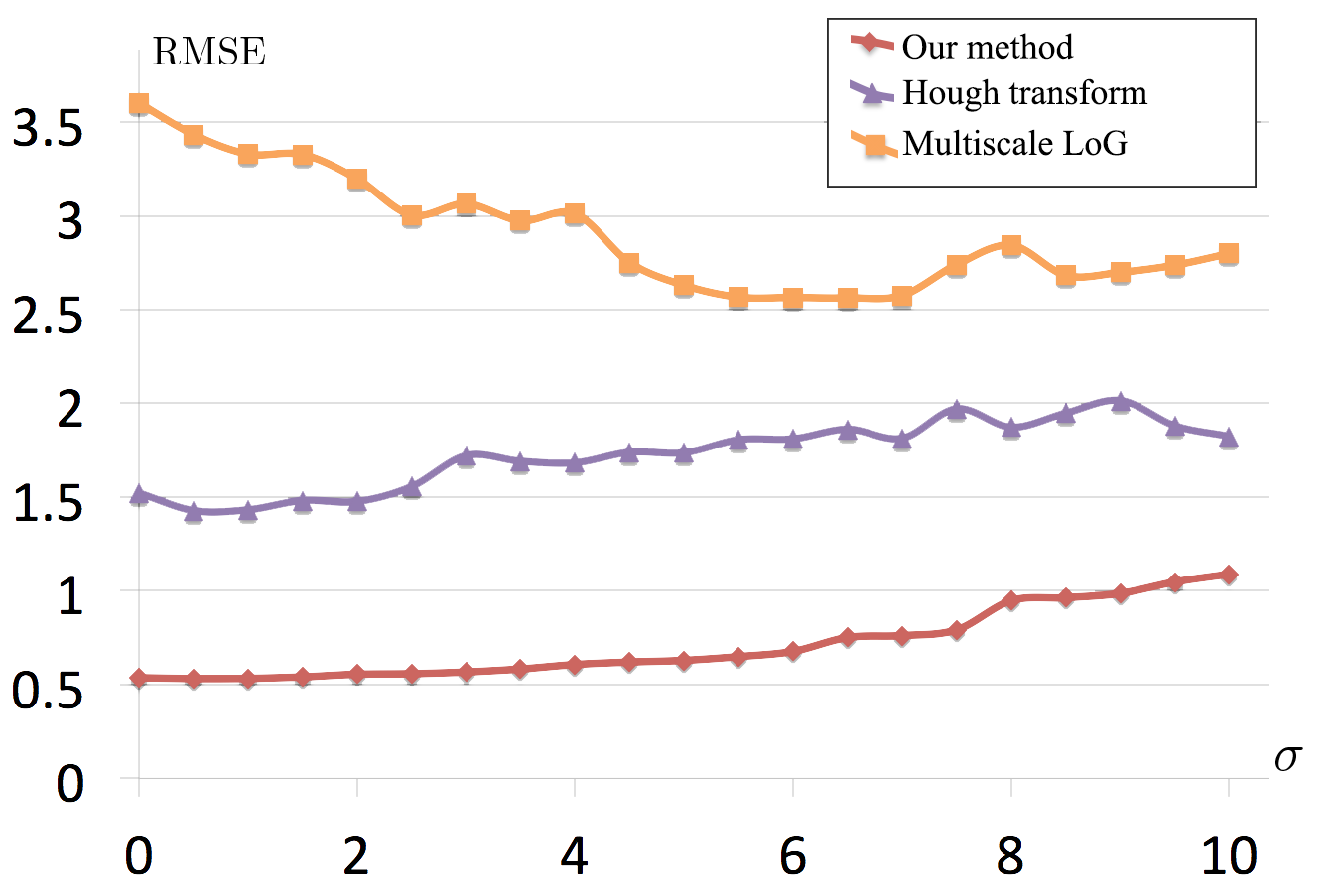}}\\
	\caption{Position and radius estimation error (in pixels) in a sense of root mean square error (RMSE) under  isotropic Brownian
	motion (background signal), as a function of background standard
	deviation ($\sigma$).}	
	\label{im:RMSECloud}	
\end{figure}

In fluorescence microscopy, the presence of background signal (autofluorescence) is typical. We test the robustness of our algorithm to this phenomena. To generate our test images, we use the results of Sage et al. \cite{Sage}, who gave an experimental validation on the spectral power density of fluorescence-microscopy images, claiming that it is isotropic. The corresponding fractional Brownian motion (fBm) is described in \cite{Vehel}. Thus, we represent the power density function by $\|\omega\|^{-s}$, where $\omega$ is the radial spatial frequency and $s$ is the fractal exponent. 

As in the previous case, we generate a series of test images (of size $\left(1,\!000 \times 1,\!000 \right)$), where
we control the location of the spots and their radii. We intend to detect different spots, with radii varying between 8 and 40. We allow overlap between neighboring spots, by at most 10 pixels. To make the detection more challenging, we add some isotropic background signal (fBm), with a mean of zero and a standard deviation ranging from 0 to 10. An illustration of the test images can be seen in Figure \ref{im:series}, along with particular results. We set the running time of the corresponding algorithms to the same order of magnitude (favoring the competing ones). The Jaccard index is presented in Figure \ref{im:JaccardCloud}, and the RMS errors in Figure \ref{im:RMSECloud}. 
Based on the graphs, we confirm that our algorithm has a clear advantage in terms of accuracy in the presence of significant background. 

\subsection{Biological Data}
\begin{figure}[!t]
\centering
	\subfloat{\includegraphics[width=.48\textwidth]{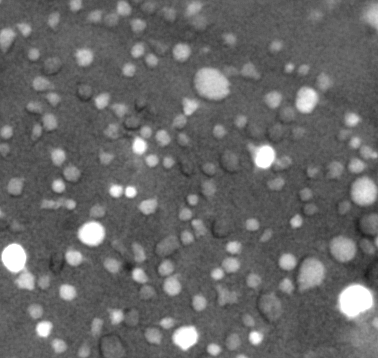}}
	\quad
	\subfloat{\includegraphics[width=.48\textwidth]{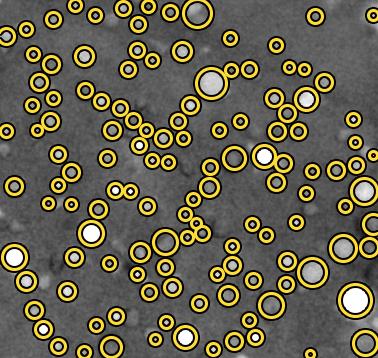}}\\
	\caption{From left to right: cells in fluorescence microscopy; corresponding detections.}	
	\label{im:fluo}	
\end{figure}
\begin{figure}[!t]
\centering
	\subfloat{\includegraphics[width=.48\textwidth]{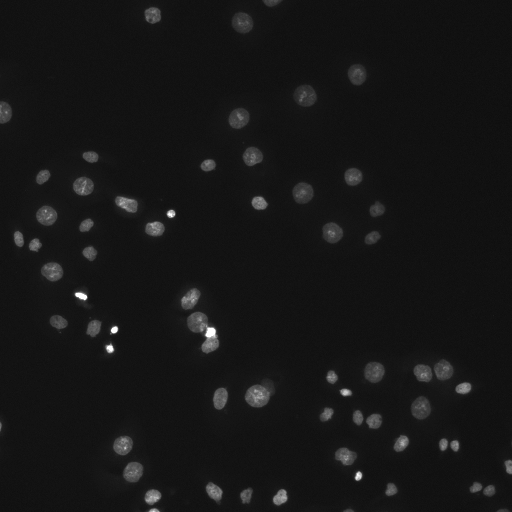}}
	\quad
	\subfloat{\includegraphics[width=.48\textwidth]{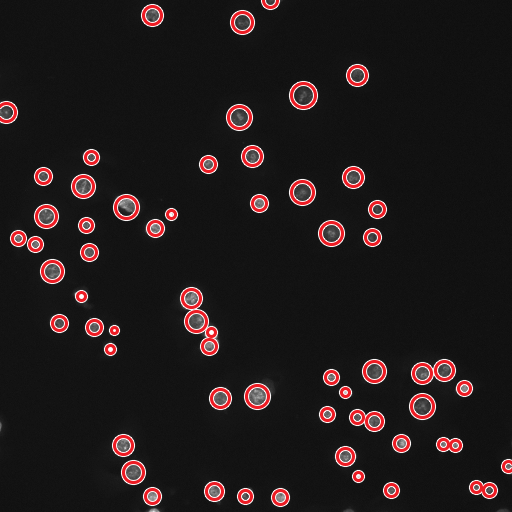}}\\
	\caption{From left to right: human colon-cancer cells; corresponding detections.}	
	\label{im:cancer}	
\end{figure}

%Most medical and biological applications relying on cell cultures (e.g. bloodwork) seek concentration information. The amount of blood cells (red or white) per unit of volume provides information on the general health of the patient; its abnormality can indicate the presence of specific diseases. For clinical purposes, the concentration of bacteria and pathogens is also frequently measured. The counting of cells  for image statistics has a fundamental role in molecular biology as well. Cell statistics provide information on the the growth rate of microorganisms (that correspond to their division to new cells), or it can be used to determine how to adjust some reagents. For clinical purposes, the accurate characterization of circular structures plays a crucial role as well. For example, in cancer research, the segmentation, identification, and precise cell count of cancer cells is fundamental.  

As a final illustration, we present results on two biological micrographs. Figure \ref{im:fluo} features cells in fluorescence microscopy. The detections and radius measurements are accurate, despite the heavy background signal and the fact that the intensity of the cells are varying. 
Figure  \ref{im:cancer} visualizes human HT29 colon-cancer cells \footnote{http://www.broadinstitute.org/bbbc/BBBC008/}. Our algorithm works well in this case as well.

\section{Conclusion}

In this paper, we presented a general construction of adaptable tight wavelet frames, focusing on scaling operations. We applied our wavelet-based framework to detect and estimate the scale of circular structures in images. 
The attractive features of our algorithm are (i) our wavelets can be scaled on a quasi continuum without significant computational overhead; (ii) robustness; and (iii), speed.  The effectiveness of our approach in practical applications was demonstrated on synthetic and real biological data in the presence and absence of background signal.

\bibliographystyle{siam}
\bibliography{main}

\end{document}